\documentclass{article} % For LaTeX2e
\usepackage{my_stylefile,times}

\usepackage{hyperref}
\usepackage{url}
\usepackage{graphicx}

\usepackage{amsmath, amsthm, amssymb}
\usepackage{mathtools}
\mathtoolsset{showonlyrefs}	
\usepackage[nameinlink, capitalise, noabbrev]{cleveref}

\makeatletter
\AddToHook{cmd/appendix/before}{\def\cref@section@alias{appendix}\def\cref@subsection@alias{appendix}}
\makeatother

\usepackage{todonotes}
\usepackage{bbm}
\usepackage{subcaption}
\usepackage{titletoc}
\usepackage{tikz}
\usepackage{enumitem}
\usepackage{wrapfig}

\theoremstyle{plain}
\newtheorem{theorem}{Theorem}

\newtheorem{proposition}{Proposition}[section]
\newtheorem{lemma}{Lemma}[section]

\theoremstyle{remark}
\newtheorem{remark}{Remark}[section]

\theoremstyle{definition}

\newtheorem{assumption}{Assumption}
\crefname{assumption}{Assumption}{Assumptions}

\DeclareMathOperator{\sign}{sign}
\DeclareMathOperator{\dist}{dist}

\DeclareMathOperator*{\esssup}{ess\,sup}
\DeclareMathOperator*{\essinf}{ess\,inf}

\DeclareMathOperator{\Per}{Per}
\DeclareMathOperator{\TV}{TV}
\DeclareMathOperator{\Lip}{Lip}
\DeclareMathOperator{\Rad}{Rad}
\DeclareMathOperator{\supp}{supp}
\newcommand{\R}{\mathbb{R}}
\newcommand{\N}{\mathbb{N}}
\newcommand{\IE}{\mathbb{E}}
\newcommand{\B}{\mathcal{B}}

\renewcommand\L{\mathcal L}

\newcommand{\de}{\,\mathrm{d}}
\newcommand{\eps}{\varepsilon}

\title{Explaining the Saliency Map Sparsity\\of Adversarially-Trained Neural Networks}

\author{
Yannick Lunk \\
Institute of Mathematics\\
University of Würzburg, Germany\\
\texttt{yannick.lunk@uni-wuerzburg.de}
\And
Atell Yehor Krasnopolsky\\
Technical University of Munich, Germany\\
\texttt{atell.krasnopolskyi@tum.de}
\And
Damien Garreau\\
Institute of Computer Science, CAIDAS\\
University of Würzburg, Germany\\
\texttt{damien.garreau@uni-wuerzburg.de}
\And 
Leon Bungert\\
Institute of Mathematics, CAIDAS\\
University of Würzburg, Germany\\
\texttt{leon.bungert@uni-wuerzburg.de}
}

\iclrpreprintcopy 

\begin{document}

% \listoftodos
% \setcounter{page}{0}
% \clearpage

\maketitle

%%%%%%%%%%%%%%%%%%%%%%%%%%%%%%%%%%%%%%%%%%%%%%%%%%%%%%%%%%%%%%%%%%

\begin{abstract}
Understanding why deep neural networks make a given prediction is of great importance for their safe deployment.  
In computer vision, saliency maps, which highlight the image region most influential for a prediction, remain a widely-used form of explanation. 
An empirical observation is the apparent sparsity of gradient saliency maps of adversarially-trained neural networks.
In this paper, we propose a theoretical explanation of this phenomenon for two-layer ReLU networks. 
We build on the established equivalence of adversarial training to the minimization of the empirical risk with weight-decay penalization and an added adversarial total variation term -- valid for certain loss functions.
As the number of data points and neurons grows and the regularization parameters are sent to zero at appropriate rates, we prove that minimizers converge to a Bayes classifier with minimal gradient and Barron norm.
Sparsity appears since for adversarial training with $\ell_\infty$-attacks the gradient norm is anisotropic and favors axis-aligned / sparse gradients. 
We illustrate our theoretical findings experimentally by evaluating the gradient $\ell_1$-norm and thresholded sparsity of naturally versus adversarially trained models.
\end{abstract}

% DGA: if space (too technical)
% The weight decay models a common design choice of machine learning optimizers and, furthermore, is necessary to ensure compactness of minimizers. 
%over the class of shallow ReLU networks, 

%%%%%%%%%%%%%%%%%%%%%%%%%%%%%%%%%%%%%%%%%%%%%%%%%%%%%%%%%%%%%%%%%%

\section{Introduction}

Despite the immense success of deep neural networks, two fundamental challenges continue to limit their deployment in high-stakes domains: their vulnerability to small, imperceptible adversarial perturbations \citep{goodfellow2015explaining}, and the ``black-box'' nature of their decision-making. 
Addressing the former, \emph{adversarial training} \citep{madry2018toward}, which mitigates the problem by incorporating adversarial examples in the training, has emerged as one of the most successful defenses and  maintained its relevance even in the era of large language models~\citep{xhonneux2024efficient}. 

To address the latter, a wide range of methods have been proposed in recent years under the umbrella of \emph{Explainable AI (XAI)} \citep{linardatos_explainable_2020}. 
Among them, local feature importance methods aim to identify the input features that are the most influential for a model's decision at a given point. 
In computer vision, a prominent example is given by saliency maps based on the gradient $\nabla f_\theta(x)$ of the model with respect to the input. 
The most important features can then be easily visualized as an image overlay, see \cref{fig:sensitivity_maps_varying_eps}. 
Motivated by the observation that plain input gradients are often noisy and visually hard to interpret, a number of extension have been proposed, notably gradient times input \citep{pmlr-v70-shrikumar17a}, integrated gradient  \citep{pmlr-v70-sundararajan17a}, and SmoothGrad \citep{smilkov2017smoothgrad}. 
A line of work also developed around specific architecture, such as GradCAM for convolutional neural networks \citep{selvaraju2017grad}. 

Previous work has shown empirically that adversarial training not only increases robustness of the trained model but also \textbf{leads to sparse, perceptually aligned, and hence more interpretable gradient saliency maps} \citep{zhang2019interpreting, etmann19connection, chalasani2020concise}; see \cref{fig:sensitivity_maps_varying_eps}. 
This observation is surprising, since adversarial training was introduced to improve robustness against worst-case perturbations, with no explicit objective related to interpretability, or even the gradient structure of the predictor. 

%%%%%%%%%%%%%%%
% MAIN RESULT %
%%%%%%%%%%%%%%%

In this paper, we provide a theoretical explanation for this phenomenon.  
The starting point of our journey is to notice that adversarial training 
(with the $\ell_1$-loss)
can be reformulated as empirical risk minimization with an additional 
% nonlocal
total variation (TV) term depending on the adversarial budget~$\eps$ \citep{bungert2023geometry}. 
Focusing on two-layer neural networks, 
%and adding a weight decay regularization term, 
we show that, in the limit of infinite width and vanishing~$\eps$, the limiting optimization problem 
reduces to finding a Bayes classifier which minimizes a balance of the average dual norm of the classifier's input gradients and a measure of its complexity. 
In particular, for the $\ell_\infty$-bounded perturbations typically used in adversarial training, this yields an $\ell_1$-norm penalty of the input gradients. 
Informally, adversarial training in this regime selects a Bayes classifier that balances accuracy, 
% Barron
complexity, and average $\ell_1$-norm of its gradients. 
Since $\ell_1$-type regularization is well-known to promote sparsity, this provides a mathematical explanation for the sparse saliency maps observed in adversarially trained models. 
We illustrate this result in a simple setting in \cref{fig:toy_model}, see \cref{sec: additional numerical experiments} for more details. 

%%%%%%%%%%%%%%%%%%%%%%%%%%%%%%%%%%%%%%%%%%%%%
% TECHNICAL TOOLS / TECHNICAL CONTRIBUTIONS %
%%%%%%%%%%%%%%%%%%%%%%%%%%%%%%%%%%%%%%%%%%%%%

Our main technical tool is Gamma-convergence (see~\citet{braides2002gamma} for an introduction), which is key in showing that the empirical risk minimizers converge to minimizers of the limit problem. 
Existing results which we review in \cref{sec: related_work}
are obtained in an unrealistic non-parametric setting with a population risk, whereas in practice one uses neural networks trained on finite datasets. 
In this work, we bridge this gap by embedding our problem into the homogeneous Barron space \citep{park2023minimum} and deriving generalization bounds and quantitative approximation results to transition from the idealized and continuous theory to the empirical and parametric setting. 
% [for safe-keeping]
% A technical challenge in the transition from continuous theory to the empirical setting is controlling the gap between the empirical and population-level TV functionals which we achieve by leveraging generalization bounds. 
%We then analyze the asymptotic behavior of adversarial training as the number of data points and the network width tend to infinity while the adversarial budget and the weight decay parameter go to zero at appropriate coupled rates.
%Employing Gamma-convergence techniques in the spirit of \citet{bungert2024gamma} and using quantitative approximation results for Barron functions, we show that the limiting optimization problem contains a TV-type gradient regularization term involving the dual norm of the norm used to constrain the adversarial perturbations.
%Hence, since adversarial examples are typically restricted to an $\ell_\infty$ ball, the limiting TV functional corresponds to an expected $\ell_1$ norm of the input gradients.
%As a result of this asymptotic analysis we show that the solutions of adversarial training converge to a Bayes classifier which minimizes a balance of the $\ell_1$ norm of the classifier's input gradients and its Barron complexity, see \cref{fig:toy_model} for an illustration on a dataset with two features.
%This $\ell_1$ gradient regularization promotes sparse input gradients, providing a mathematical explanation for the sparse saliency maps observed in robust models.

%%%%%%%%%%%%%%%%%%%%%%%%%%%%%%%%%%%%%%%%%%%%%%%%%%%%%%%%%%%%%%%

\begin{figure}[t!]
\centering
    \begin{subfigure}[b]{0.19\linewidth}
        \centering
        \includegraphics[width=\linewidth]{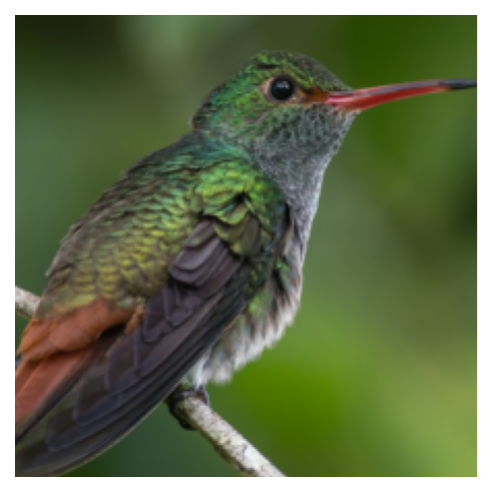}
        \caption{Original}
        \label{fig:sens_orig}
    \end{subfigure}
    \hfill
    \begin{subfigure}[b]{0.19\linewidth}
        \centering
        \includegraphics[width=\linewidth]{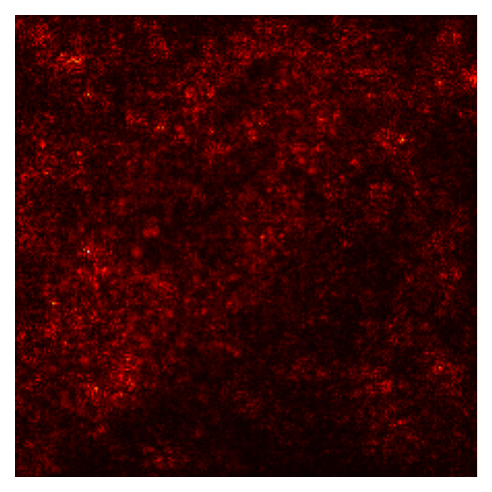}
        \caption{Natural training}
        \label{fig:sens_clean}
    \end{subfigure}
    \hfill
    \begin{subfigure}[b]{0.19\linewidth}
        \centering
        \includegraphics[width=\linewidth]{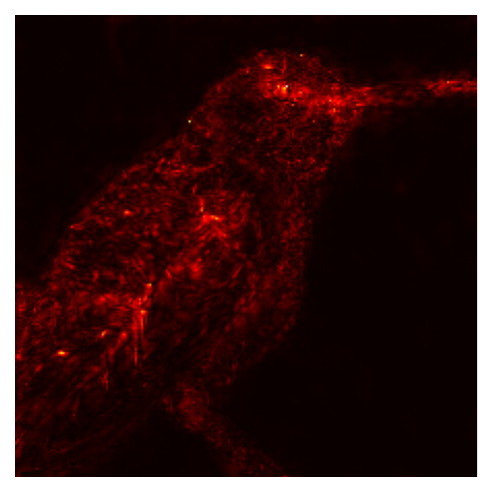}
        \caption{$\eps=0.5/255$}
        \label{fig:sens_eps0.5}
    \end{subfigure}
    \hfill
    \begin{subfigure}[b]{0.19\linewidth}
        \centering
        \includegraphics[width=\linewidth]{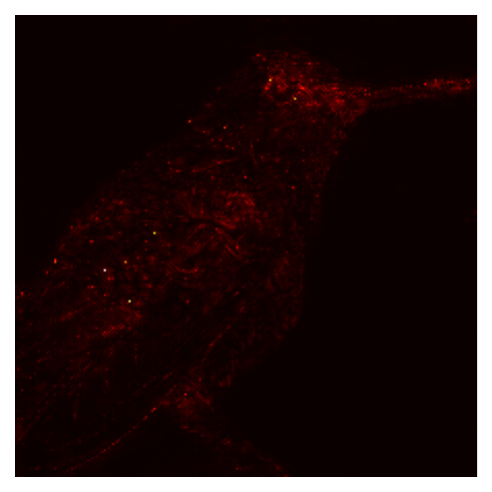}
        \caption{$\eps=1/255$}
        \label{fig:sens_eps1.0}
    \end{subfigure}
    \hfill
    \begin{subfigure}[b]{0.19\linewidth}
        \centering
        \includegraphics[width=\linewidth]{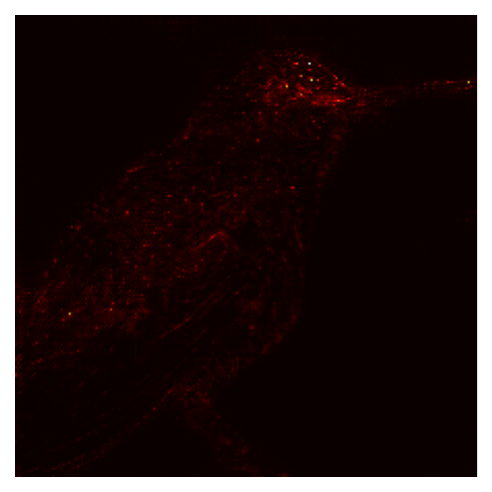}
        \caption{$\eps=2/255$}
        \label{fig:sens_eps2.0}
    \end{subfigure}
    \caption{Comparison of saliency maps for a naturally trained ResNet50 on ImageNet versus adversarially trained models across varying $\ell_\infty$ perturbation budgets $\eps$. As predicted by \cref{thm: main}, with decreasing adversarial budget~$\eps$ the gradient saliency maps converge to a sparse one whereas natural training (corresponding to $\eps=0$) produces a noisy saliency map.}    \label{fig:sensitivity_maps_varying_eps}
\end{figure}

%%%%%%%%%%%%%%%%%%%%%%%%%%%%%%%%%%%%%%%%%%%%%%%%%%%%%%%%%%%%%%%

Our main contributions are as follows:
\begin{itemize}[itemsep=0pt, topsep=0pt]
\item \cref{thm: main} states that as the number of data points and the width of the neural networks go to infinity and the adversarial budget and the weight decay coefficient go to zero at specific coupled rates, every sequence of empirical minimizers admits a subsequence that converges to a Bayes classifier with optimal balance of its Barron semi-norm with a TV functional of the form $\int \|\nabla f(x)\|_\star \de\mu(x)$, where $\|\cdot\|_\star$ is the dual norm to the one restricting adversarial attacks. 
From this, we deduce an explanation for the observed sparsity of saliency maps. 
    \item We support our theoretical findings with numerical experiments, demonstrating that adversarial training systematically reduces the average $\ell_1$-norm of the input gradients. 
    Furthermore, by evaluating thresholded gradient sparsity, we provide a quantitative empirical validation of the observations of \citet{zhang2019interpreting,etmann19connection} and confirm that this implicit $\ell_1$-regularization translates to the sparse saliency maps.
\end{itemize}

%%%%%%%%%%%%%%%%%%%%%%%%%%%%%%%%%%%%%%%%%%%%%%%%%%%%%%%%%%%%%%%
 
\begin{figure}[htb]
\centering%
\begin{tikzpicture}
    % Contour plot
    \node[anchor=south west, inner sep=0] (plot) at (0,0) {
 \includegraphics[width=\linewidth,trim=0.4cm 2cm 0.3cm 2cm,clip]{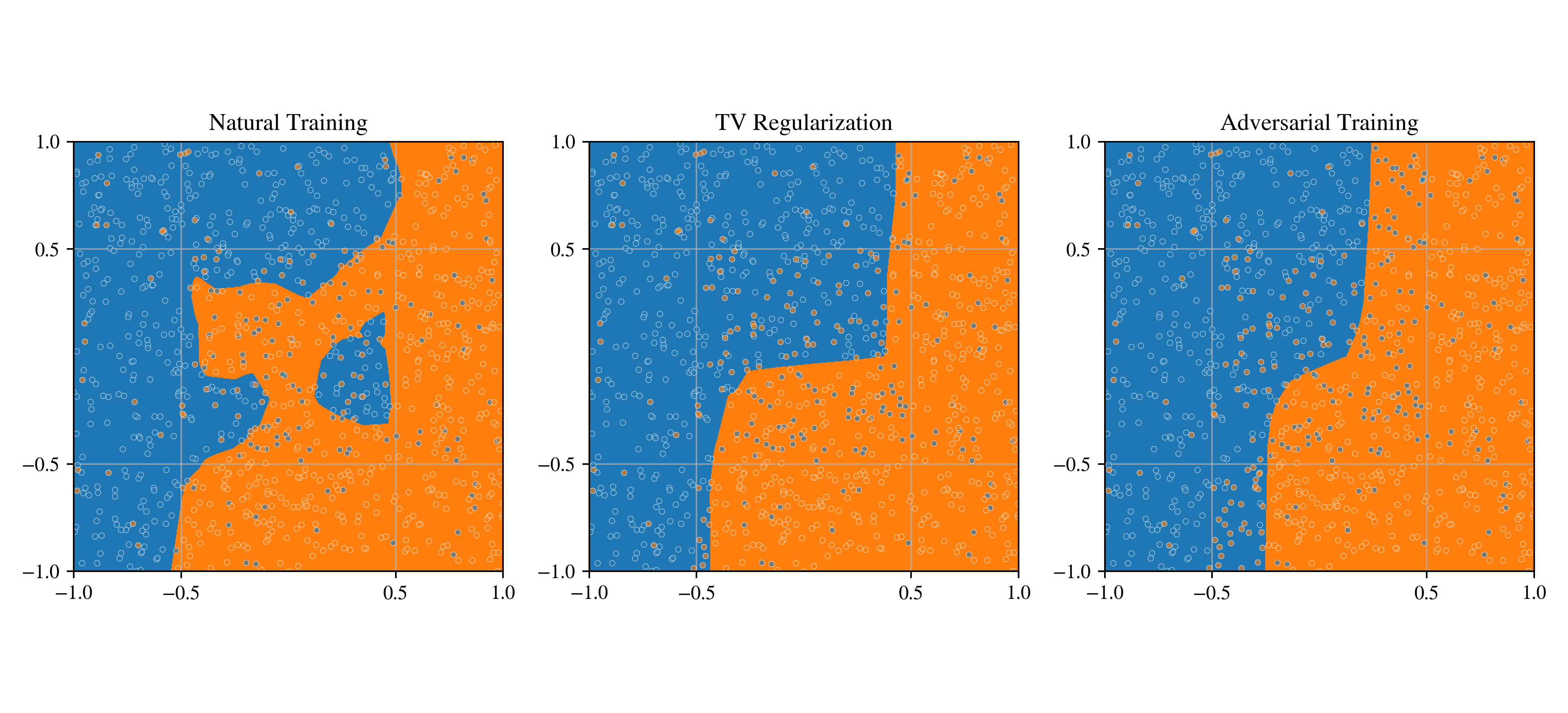}
    % };
%\includegraphics[width=\linewidth]{Figures/toy_cropped.png}
    };
% set these to 0 to revert to original figure
    % Gradient arrows
    \draw[->, very thick]
        (2.8,2.3) -- (2.4,2.7)
        node[midway, above right] {$\nabla f(x)$};    

    \draw[->, very thick]
        (7.95,2.8) -- (8.5,2.78)
        node[midway, above] {$\nabla f(x)$};    
    
    \draw[->, very thick]
        (12.25,2.8) -- (12.8,2.78)
        node[midway, above right] {$\nabla f(x)$};    

    % mark the points x
    \fill (2.8,2.3) circle (2pt)
        node[below left] {$x$};

    \fill (7.95,2.8) circle (2pt)
        node[below] {$x$};

    \fill (12.25,2.8) circle (2pt)
        node[below] {$x$};    
\end{tikzpicture}
 % \vspace{-2em}
    \caption{Decision boundaries and training data for a data distribution with non-uniquely determined Bayes classifier in the strip $[-0.5,0.5]^2$:
    Natural training (left) has a highly oscillatory decision boundary arising from fitting the training data. 
    Nonlocal TV regularization (center, analyzed here) and adversarial training (right) have a decision boundary with small anisotropic perimeter and exhibit approximate gradient sparsity, as indicated by the axis-parallel parts of the decision boundary. See \cref{sec: additional numerical experiments} for more details.}
    \label{fig:toy_model}
\end{figure}

%%%%%%%%%%%%%%%%%%%%%%%%%%%%%%%%%%%%%%%%%%%%%%%%%%%%%%%%%%%%%%%%%%

\section{Setting}
\label{sec: setting}

We study a binary classification problem, where the distribution of the training pairs $(x,y)\in\Omega\times\{0,1\}$ is described by some probability measure $\mu$ on $\Omega\times\{0,1\}$. In our setting, we assume that $\Omega\subset\R^d$ is a bounded domain with Lipschitz boundary. 
We define $\mu_i \coloneq \mu(\cdot\times\{i\})$ for $i\in\{0,1\}$, so $\frac{\mu_i}{\mu_i(\Omega)}$ is the probability distribution of the data points corresponding to the label $i$. 

\paragraph{Two-Layer Neural Networks.} 
In this work, we consider two-layer neural networks with $m$ hidden neurons of the form 
\begin{align}\label{eq:neural network}
    f_m(x) \coloneq b_0 + \sum_{i=1}^m a_i\sigma(w_i\cdot x+b_i),
    \qquad
    x\in\R^d,
\end{align}
for $\theta \coloneq (a,W,b) \in\R^m\times\R^{m\times d}\times\R^{m+1} =: \Theta_m$, where $w_i$ is the $i$-th row of the matrix $W$ and $\sigma$ denotes the ReLU activation function defined as $\sigma(t)\coloneq \max\{0,t\}$.
The matrix $W$ collects the weights of the hidden layer, while $(b_1,\dots,b_m)$ forms the bias vector. 
Additionally, $a$ denotes the output weights and $b_0$ acts as the output bias.
Often, when training neural networks, the weight decay regularizer
% \begin{equation}
% \label{eq:reg}
\(
R_{WD}(\theta) \coloneq \frac{1}{2} \sum_{i=1}^m \big(a_i^2 + \|w_i\|_2^2\big)
% \end{equation}
\)
is added to the objective, where $\|\cdot\|_2$ denotes the Euclidean norm.
Importantly, this functional does not control the bias vector~$b$, which is typical. 

%%%%%%%%%%%%%%%%%%%%%%%%%%%%%%%%%%%%%%%%%%%%%%

\paragraph{Barron space.}
Taking the infinite width limit \citep{bach2017breaking, e2022barron} of such neural networks, one ends up with functions of the form 
% \begin{align}
%     f_\pi(x) \coloneq \int_{\R\times\R^d\times\R} a\sigma(w\cdot x+b)\de \pi(a,w,b)
% \end{align}
\begin{align}\label{eq: homogeneous barron}
    f_{\pi,y}(x) \coloneq y + \int_{\R\times\R^d\times\R}a\big(\sigma(w\cdot x+b)-\sigma(b)\big)\de\pi(a,w,b),\qquad x\in\R^d,
\end{align}
where $y\in\R$ captures the function's value at the origin, and $\pi$ is a probability measure on $\R\times\R^d\times\R$. 
With a slight abuse of notation, the counterpart of the weight decay regularizer is given by
\[
    R_{WD}(\pi)\coloneq \frac{1}{2}\int_{\R\times\R^d\times\R}\big(a^2+\|w\|_2^2\big)\de\pi(a,w,b). 
\]
For a given function $f$ written as~\eqref{eq: homogeneous barron}, its representation via a measure $\pi$ is not unique and hence $R_{WD}(\pi)$ is not well-suited to measure the complexity of the function.
Instead, we define the Barron semi-norm $[f]_\B$ as the infimum of the weight decay regularizer over all possible representations, 
\begin{align}
    [f]_{\B(\Omega)} \coloneq \inf_{f_{\pi,y} = f \text{ on }\Omega} R_{WD}(\pi).
\end{align}
The \emph{homogeneous Barron space} $\B(\Omega)$ is then defined as the set of all functions that admit a representation $f = f_{\pi,y}$ with a finite Barron semi-norm, 
\begin{align}
    \B(\Omega)\coloneq \{f\colon\Omega\to\R\,\vert\,[f]_{\B(\Omega)} < \infty\}. 
\end{align}
For more details on this space and the previous definitions, we refer to \cref{sec: barron space}, \citet[Appendix C]{park2023minimum} and \citet{Ongie2020A}. 

\begin{remark}\label{rmk: nn}
    Neural networks of the form \eqref{eq:neural network} are contained in $\B(\Omega)$. Setting $y = b_0 + \sum_{i=1}^m a_i\sigma(b_i)$ and defining the probability measure $\pi = \frac{1}{m}\sum_{i=1}^m \delta_{\sqrt{m}(a_i,w_i,b_i)}$ yields $f_{\pi,y} = f_m$.
    %The scaling factor $\sqrt{m}$ compensates for the $1/m$ normalization of $\pi$ due to the positive $1$-homogeneity of the ReLU activation.
    Consequently, we get that $[f_m]_\B \leq R_{WD}(\pi) = R_{WD}(\theta)$.
\end{remark}

%%%%%%%%%%%%%%%%%%%%%%%%%%%%%%%%%%%%%%%%%%%%%%%%%%%%%%%%%%%%%%%%%%%%%%%

\paragraph{Adversarial Training.}
\emph{Adversarial training} aims to minimize the adversarial population risk which, in our binary classification setting, is given by
\begin{align}\label{eq:AT}
    \inf_{f\in\mathcal{H}} \int_{\Omega\times\{0,1\}}
    \sup_{\tilde x \in B_\eps(x)}\ell(f(\tilde x),y)\de\mu(x,y),
\end{align}
where $\ell:\R\times\{0,1\}$ is a loss-function, $\eps>0$ is the adversarial budget, and minimization takes place over a certain class of functions $\mathcal{H}$. 
Here and in the following, the balls $B_\eps(x)$ are taken with respect to an \emph{arbitrary} norm $\|\cdot\|$ on $\R^d$ (not necessarily the Euclidean norm) and by $\|\cdot\|_\star$, we denote its dual norm, defined via $\|y\|_\star\coloneq\sup_{\|x\|\leq1}\langle y,x\rangle$. 
In practice, the norm considered is often the $\ell_\infty$-norm \citep{madry2018toward, croce2021robustbench}, leading to the $\ell_1$-norm as the dual norm.  

% [from an older version, for safe-keeping]
% To tackle \eqref{eq:AT} computationally, one employs a stochastic gradient descent type scheme based on i.i.d. training data $(x,y)\sim\mu$ which are attacked by applying a few steps of projected gradient ascent to approximate the supremum.
% The worst-case structure of the objective function in \eqref{eq:AT} encourages robustness of the optimal classifier against attacks in an $\eps$-ball. 
% In fact, if a classifier achieves zero adversarial risk on a dataset, then it is certifiably robust against such attacks of this very dataset.\todo{could be removed}

%%%%%%%%%%%%%%%%%%%%%%%%%%%%%%%%%%%%%%%%%%%%%%

\paragraph{Total Variation Regularization.}
Let $\mathcal{H}$ be an hypothesis class of functions from $\Omega$ to $\R$. 
As is typical, 
%for any $f\in \mathcal{H}$ 
we define the \emph{population risk} as 
\begin{equation}
\label{eq:population-risk}
\mathcal{R}(f) \coloneq \int_{\Omega\times\{0,1\}}\ell(f(x),y)\de\mu(x,y),
\qquad f\in\mathcal{H} .
\end{equation}
Before introducing our primary empirical objective, we introduce the following motivating population-level problem
\begin{align}
\label{eq:Loss+TV}
\inf_{f\in\mathcal{H}} 
% \bigl\{ 
\mathcal{R}(f)
%\int_{\Omega\times\{0,1\}} \ell(f(x),y)\de\mu(x,y)
+\eps\TV_\eps(f;\mu)
% \bigr\} 
\, ,
\end{align}
where $\TV_\eps(f;\mu)$ is a \emph{nonlocal total variation} given by
\begin{align}
\label{eq: def TVeps}
    \TV_\eps(f;\mu) \coloneq 
    %\frac{1}{\eps}
    \int_\Omega 
    %\left(\sup_{B_\eps(x)\cap\Omega} f - f(x)\right) 
    \frac{\sup_{B_\eps(x)\cap\Omega} f - f(x)}{\eps}
    \mu_0(x)\de x + 
    %\frac{1}{\eps}
    \int_\Omega 
    %\left(f(x)-\inf_{B_\eps(x)\cap\Omega} f\right)
    \frac{f(x)-\inf_{B_\eps(x)\cap\Omega} f}{\eps}
    \mu_1(x)\de x.
\end{align}
Note that the objective in \eqref{eq:Loss+TV} is very similar to the TRADES method~\citep{zhang2019theoretically} and both coincide with adversarial training for the $\ell_1$-loss~\citep{bungert2023geometry}.
Intuitively, the total variation $\TV_\eps$ is closely connected to a nonlocal, $\eps$-dependent surface content of the decision boundary, and
quantifies the adversarial robustness of a classifier~\citep{bungert2024gamma}. 
One should think of~\eqref{eq:Loss+TV} as \textbf{loosely 
equivalent to population-level adversarial training} with budget~$\eps$, with strict mathematical equivalence for specific choices of loss functions and hypothesis classes. 
Our motivation for working with \eqref{eq:Loss+TV} instead of standard adversarial training \eqref{eq:AT} is that the former is amenable to our asymptotic analysis thanks to deep mathematical properties of the nonlocal total variation~$\TV_\eps$.

%%%%%%%%%%%%%%%%%%%%%%%%%%%%%%%%%%%%%%%%%%%%%%%%%%%%

\paragraph{Key optimization problem.}
In our setting, we are given $(X_1,Y_1), \dots, (X_n,Y_n) \sim \mu$, $n$ i.i.d. samples, and restrict our attention to $f=f_\theta$, where $f_\theta$ with $\theta\in\Theta_{m_n}$ belongs to the class of two-layer neural networks with ReLU activation and $m_n\in\N$ neurons \eqref{eq:neural network}. 
The population risk \eqref{eq:population-risk} is approximated, as usual, by the \emph{empirical risk} 
%with respect to these samples 
\begin{equation}
\label{eq: empirical risk}
\widehat{\mathcal{R}}_n(f) \coloneq \frac{1}{n}\sum_{i=1}^n \ell(f(X_i),Y_i)
\, ,
% \qquad
\end{equation} 
whereas the empirical counterpart of $\TV_\eps(f;\mu)$ is given by
\begin{equation}\label{eq: empirical TVeps}
\widehat{\TV}_{\eps_n,n} (f;\mu) 
\coloneq 
    \frac{1}{n}
\left(\sum_{i\,:\,Y_i=0 }
%\left(\sup_{B_{\eps_n}(X_i)\cap\Omega}f - f(X_i)\right) 
\frac{\sup_{B_{\eps_n}(X_i)\cap\Omega}f - f(X_i)}{\eps_n}
\!+ 
\!\sum_{i\,:\,Y_i=1}
%\left(f(X_i) - \inf_{B_{\eps_n}(X_i)\cap\Omega} f\right)
\frac{f(X_i) - \inf_{B_{\eps_n}(X_i)\cap\Omega} f}{\eps_n}
\right).
\end{equation} 
With these definitions in hand, the natural finite-sample version of \eqref{eq:Loss+TV} is 
\begin{equation}
\label{eq:intermediate-sample-tv}
\inf_{\theta\in\Theta_{m_n}} \widehat{\mathcal{R}}_n(f_\theta) + \eps_n \widehat{\TV}_{\eps_n,n}(f_\theta;\mu) 
\, .
\end{equation} 
% DGA: 'parameters might blow up' in the next sentence is too vague. what happens exactly? 
% LB: The space of ReLU networks is not closed under various notions of convergence, in particular, the weak-* topology on $L^\infty, under which TV_\eps is lower semicontinuous, see \citet{bungert2023geometry}. ;-) 
%To fit discontinuous training data, as arises in classification, you have to steepen the networks. I would not discuss any of this here.
% DGA: so the solutions for the population version, but we are now in the empirical setting, maybe a bit confusing for the reader? (unless this also happens in the finite sample world)
%LB: that is a good point, as far as I know there is no existence proof in the empirical setting either. Another motivation for weight decay is compactness. For us this is even more important.
Minimization problems over a space of neural networks usually do not possess solutions since parameters might have to blow up to fit the data.
Therefore, we add the usual weight regularization term $R_{WD}(\theta)$ defined previously which ensures existence as well as a mathematical compactness property in the Barron space $\B(\Omega)$, making sure that sequences of minimizers $(\theta_n)_{n\in\N}$ of problem \eqref{eq: full objective} below do not diverge. 
We want to stress that, while the weight decay term is usually not written explicitly as part of the objective in the practice of adversarial training, it is often implicitly present in the optimization procedure.\footnote{See, e.g., the adversarial training code of the Madry lab: \url{https://github.com/MadryLab/cifar10_challenge}} 
Letting the adversarial budget $\eps_n>0$ and the weight decay parameter $\lambda_n>0$ depend on~$n$, we arrive at the following optimization problem, which is at the core of our study: 
\begin{equation}
\label{eq: full objective}
\inf_{\theta\in\Theta_{m_n}} \widehat{\mathcal{R}}_n(f_\theta) + \eps_n \widehat{\TV}_{\eps_n,n}(f_\theta;\mu) + \lambda_n R_{WD}(\theta)
\, .
\end{equation}
%where $f_\theta$ denotes the neural network with $m_n$ neurons parametrized by the parameters $\theta\in\Theta_{m_n}$. 
In the same spirit as the previous paragraph, one should think of \eqref{eq: full objective} as loosely
equivalent to finite-sample adversarial training with budget~$\eps_n$ and weight decay~$\lambda_n$. 
We show in \cref{fig:toy_model,fig:robustness_bars,fig:l2_classifiers} for the toy dataset and in \cref{fig:robustness_bars_CIFAR10,fig:sparsity_comparison_CIFAR10} for CIFAR-10 that empirically they produce very similar results, with even a slight advantage for \eqref{eq: full objective} in terms of robust accuracy and sparsity on CIFAR-10. 

%%%%%%%%%%%%%%%%%%%%%%%%%%%%%%%%%%%%%%%%%%%%%%%%%%%%%%%%%%%%%%%%

%\paragraph{Main Result.}
\section{Main result}
\label{sec: main result}

% We now turn to our main result: 
In this section, we state our main result: under suitable scalings of the parameters $\eps_n,\lambda_n \to 0$ and $m_n\to\infty$, a sequence of solutions to \eqref{eq: full objective} converges in a weak sense to a Bayes classifier with specific properties. 
To make this precise, we first define 
\begin{align}
    \TV(f;\mu) \coloneq \int_\Omega \|\nabla f(x)\|_\star (\mu_0(x)+\mu_1(x)) \de x
\end{align}
for $f\in\B$, which arises as limit of $\TV_\eps$ defined in $\eqref{eq: def TVeps}$ as $\eps\to0$.
We note that this expression is well-defined for Barron functions, since they are Lipschitz-continuous (\Cref{lem:lipschitz_barron}) and in particular differentiable almost everywhere by Rademacher's theorem. 
We additionally define the functional 
\begin{align}\label{eq: def F}
    F(f) \coloneq \begin{dcases}
        \TV(f;\mu)+\alpha [f]_\B, & f\in\B \text{ and } \mathcal{R}(f) = \inf_\B \mathcal{R}, \\
        \infty, &\text{otherwise}. 
    \end{dcases}
\end{align} 
In the whole paper, we shall assume the existence of a Barron function $f^\star\in\B$ that minimizes the population risk $\mathcal{R}$, i.e., $\mathcal{R}(f^\star) = \inf_\B\mathcal{R}$. 
Otherwise, $F\equiv\infty$ and there is nothing to show.

Minimizers of the functional $F$ are optimal classifiers (i.e., minimizers of the population risk in \eqref{eq: empirical risk}) that additionally achieve the minimal value of the sum of the total variation $\TV(f;\mu)$ and the Barron semi-norm $\alpha[f]_\B$.
Hence, adversarial training acts as selection procedure to pick a robust Bayes classifier as the adversarial budget shrinks.
In a nutshell, \cref{thm: main} below states that the limit of solutions to \eqref{eq: full objective} is a minimizer of the functional~$F$. 

We make the following assumptions on the loss function $\ell$ in \eqref{eq: empirical risk}.
\begin{assumption}\label{assum: loss}
    We assume that the loss function $\ell:\R\times\{0,1\}\to[0,\infty)$ is measurable, bounded and Lipschitz continuous in the first argument. 
    Moreover, we assume that there exist constants $c,C > 0$ such that almost surely (w.r.t. the samples) $\widehat{\mathcal{R}}_n(f)-\inf\mathcal{R} \geq c$ for all but finitely many $n\in\N$ whenever $f \geq C$ on $\Omega$ or $f \leq -C$ on $\Omega$.
\end{assumption} 
The first part of \cref{assum: loss} is satisfied for, e.g., clipped versions of the cross-entropy, hinge loss and $\ell_1$-loss. 
The second part demands that trivial constant classifiers be strictly suboptimal which is true, e.g., for $\ell(\tilde{y},y) \coloneq |\tilde{y}-y|$. 
Then, for $f$ such that $f\geq 3$, we immediately get $\widehat{\mathcal{R}}_n(f) \geq 2$ while $\inf\mathcal{R} \leq \mathcal{R}(1) \leq 1$ yielding $\widehat{\mathcal{R}}_n(f) - \inf\mathcal{R} \geq 1$. 
Similar arguments hold for $f \leq -2$ and the statement obviously remains true if the loss function is clipped at a sufficiently large value. 

Additionally, we need to make assumptions about the data distribution $\mu$ \citep[cf.][]{siethoff2025existence} which enforce the class-conditionals to be sufficiently regular and the data density to be strictly positive.
\begin{assumption}\label{assum: mu}
    We assume that the measures $\mu_0$ and $\mu_1$ are absolutely continuous with respect to the Lebesgue measure.
    By slight abuse of notation, we denote their densities by $\mu_0$ and $\mu_1$ as well and assume that $\mu_i \in C(\overline{\Omega})$ for $i\in\{0,1\}$. 
    Moreover, we assume that $\mu_1+\mu_0 > 0$ on $\overline{\Omega}$.
\end{assumption} 

With these assumptions in place, we can state our main result. 

\begin{theorem}\label{thm: main}
    Let \cref{assum: loss,assum: mu} be satisfied. 
    Moreover, let $(\eps_n)_{n\in\N}$ be a sequence of positive numbers converging to zero and $\lambda_n > 0$, $m_n\in\N$ chosen such that 
    \begin{align}
\label{eq:limit-functional}
        \lim_{n\to\infty} \left(\frac{1}{\eps_n} \sqrt{\frac{\log(n)}{n}} + \frac{1}{\eps_n\sqrt{m_n}} + \frac{1}{\lambda_n\sqrt{n}}\right) =0,  \qquad  \alpha \coloneq \lim_{n\to\infty} \frac{\lambda_n}{\eps_n} \in (0,\infty). 
    \end{align}
    Let $\theta_n$ be a solution to \eqref{eq: full objective} and $f_n = f_{\theta_n}$ be the associated neural network. 
    Then, almost surely, the sequence $(f_n)$ possesses a subsequence that converges in $L^1(\Omega)$ to a minimizer of $F$. 
\end{theorem}

Note that a similar result was shown by \citet{bungert2024gamma} and quantified by \citet{morris2026uniform}, however, in a non-parametric population risk setting where the Barron norm $[f]_\B$, that we need to control the parameters of the neural networks, does not appear. 

\begin{remark}
    We note that there exist many hyperparameter configurations satisfying the scaling assumptions.
    One may, for example, choose $m_n = n$, and $\lambda_n = \eps_n = n^{-\frac{1}{3}}$, leading to $\alpha = 1$. 
\end{remark}

\begin{remark}\label{rem:approximate_min}
    While \Cref{thm: main} is stated for minimizers of \eqref{eq: full objective}, it is actually sufficient to consider quasi-minimizers as long as the gap between the achieved objective and the optimal value vanishes sufficiently fast as $n\to\infty$. Precisely, \Cref{thm: main} also holds for $\theta_n$ with
    \(
        \mathcal{L}_n(\theta_n) - \inf_{\theta\in\Theta_{m_n}}\mathcal{L}_n(\theta) = o(\eps_n), 
    \)
    where $\mathcal{L}_n$ denotes the objective in \eqref{eq: full objective}:
    \begin{align}
    \mathcal{L}_n(\theta) \coloneq \widehat{\mathcal{R}}_n(f_{\theta}) + \eps_n\widehat{\TV}_{\eps_n,n}(f_{\theta};\mu) + \lambda_n R_{WD}(\theta).
    \end{align}
\end{remark}

\paragraph{Connection to Sparse Saliency Maps.}

The standard choice for the norm in adversarial training is the $\ell_\infty$-norm with the $\ell_1$-norm as the corresponding dual norm. 
In this particular setting, \Cref{thm: main} states that the minimizers of \eqref{eq: full objective} converge (up to a subsequence) to an optimal classifier minimizing
\begin{align}
    \int_\Omega \|\nabla f(x)\|_1 (\mu_0(x)+\mu_1(x)) \de x + \alpha [f]_\B. 
\end{align}
To interpret the effect of this regularizer, it is instructive to first consider a linear classifier $f(x) = w\cdot x + b$. 
In this regime, the gradient is constant with respect to the input ($\nabla f(x) = w$), and the integral reduces to $\|w\|_1$. 
Furthermore, its Barron semi-norm is $[f]_\B=\|w\|_2$
%\todo[inline]{double-check \\YL: I am not sure about the square and on a bounded domain one might get rid of the coefficient.\\LB: sorry, the square was a typo. I also agree with your first point, writing $f(x)=w\cdot x+b = \frac{1}{c}\sigma(c(w\cdot x + b - m))+m$ where $c>0$ is arbitrary and $m$ is chosen such that $w\cdot x+b-m>0$ for all $x\in\Omega$.\\YL: Exactly, I think on $\R^d$, we have $[f]_\B = 2\|w\|_2$ and on $\Omega$, we have $[f]_\B = \|w\|_2$. LB: Yes} 
and the regularizer reduces to the elastic-net-type regularization $\|w\|_1+\alpha\|w\|_2$ which is known to be sparsity-promoting, e.g., if the loss function is the squared $\ell_2$-loss \citep{zou2005regularization}.
\\
In the general non-linear case, the objective imposes an \textit{average} $\ell_1$-penalty of the gradient weighted by the data distribution plus the Barron semi-norm.
The former can be interpreted as group $\ell_{1,1}$-norm where a group is the array $\nabla f(x)$ for a data point $x$.
Hence, the regularizer promotes group-sparsity in the sense that for many data points the array $\nabla f(x)$ is the zero array and for the remaining data points (the ones close to or on the decision boundary) it is sparse; see \cref{fig:toy_model} for an illustration.

%%%%%%%%%%%%%%%%%%%%%%%%%%%%%%%%%%%%%%%%%%%%%%%%%%%%%%%

\paragraph{Limitations and Extensions.}
\subparagraph{Limiting Cases for $\alpha$}
We note that we restrict our analysis to the case, where $\alpha$ is positive but finite, meaning that  weight decay and adversarial regularization are asymptotically equivalent. 
%equally strong. 
On a technical level, this enables us to argue that sequences of minimizers of \eqref{eq: full objective} are compact. 
While the case $\alpha = \infty$, corresponding to $\lambda_n \gg \eps_n$ essentially boils down to the setting considered in \citet{park2023minimum} without the total variation regularization, the case $\alpha = 0$ presents technical challenges. 
Specifically, in \Cref{prop: limsup AT} one would need to approximate general $BV$ functions with neural networks and not only Barron functions. 
Bounding the approximating sequence's Barron semi-norm in this case is expected to require a dimension-dependent scaling of $\eps_n$ relative to $n$. 

%\paragraph{Multiclass Classification.}
\subparagraph{Multiclass Classification}
In a classification problem with more than two different classes, which means that labels can take values in $\{1,\dots,C\}$ for some $C\in\N$, we need to consider functions $f:\Omega\to\R^C$, and the nonlocal total variation takes the form~\citep{garciatrillos2026existence}
\begin{align}
    \TV_\eps(f;\mu) \coloneq \frac{1}{\eps}\sum_{c=1}^C \int_\Omega \left(f_c(x) - \inf_{B_\eps(x)\cap\Omega} f_c\right) \mu_c(x)\de x,
\end{align}
where $\mu_c$ describes the distribution of data points $x$ with label $c$ and $f_c$ denotes the $c$-th component of the vector-valued function $f$. 
% Obviously, the empirical counterpart is given by  
% \begin{align}
%     \widehat{\TV}_{\eps,n}(f;\mu) \coloneq \frac{1}{n\eps}\sum_{c=1}^C \sum_{Y_i=c}\left( f_c(X_i) - \inf_{B_\eps(X_i)\cap\Omega} f_c \right) 
% \end{align}
% and the limit functional in this case is 
\\
The limit functional in this case is
\(
    \TV(f;\mu) \coloneq \sum_{c=1}^C\int_\Omega \|\nabla f_c(x)\|_\star \mu_c(x) \de x
\)
for sufficiently regular functions $f$.
We expect our theory to extend to this setting, provided that the neural network approximation result (\Cref{lem:nn_approx_in_l_infty}) and the bounds for the Rademacher complexity (\Cref{lem:Rademacher bound,lem: Rademacher bound 2}) are generalized to vector-valued Barron functions.

%\paragraph{General Activation Functions.}
\subparagraph{General Activation Functions}
The specific structure of the ReLU activation function enters at different occasions into our analysis, for example, when proving lower semi-continuity of the Barron semi-norm (\Cref{lem: semi norm lower semicontinuous}).
However, we believe that the theory can be extended to general Lipschitz-continuous activation functions if one is willing to include the bias term into the regularizer which is favorable for showing compactness in the Barron space. 

%\paragraph{Deep Networks.}
\subparagraph{Deep Networks}
We only consider neural networks with one hidden layer, and these can be embedded into the Barron space. 
For deeper feedforward architectures such an analytical framework does not exist to this date.
A notable exception are deep ResNets since they can be embedded into the flow-induced Barron space \citep{e2022barron} which supports approximation by neural networks and Rademacher complexity bounds.
We expect that our results can be generalized to such architectures.

\subparagraph{Original Adversarial Training}
%\paragraph{Original Adversarial Training.}
Our objective in \eqref{eq: full objective} only coincides with the ``original'' adversarial training of \citet{madry2018toward} for the zero-one loss \citep{bungert2023geometry}. 
For other loss functions, \eqref{eq: full objective} is a variant of adversarial training that is interesting in its own right, just like the TRADES method \citep{zhang2019theoretically}.
As mentioned before, the results achieved by this method (see \cref{fig:toy_model,fig:robustness_bars,fig:robustness_bars_CIFAR10,fig:sparsity_comparison_CIFAR10}) are very similar to the results produced by adversarial training, even though we use the cross-entropy loss in these experiments. 
Extending our analysis to Madry's adversarial training seems challenging, but we expect the corresponding regularizer to have the similar form $\sum_{c=1}^C\int_\Omega\|\nabla_x\ell(u(x),c)\|_\star\mu_c(x)\de x$ \citep[cf.][]{siethoff2025existence,finlay2021scaleable}.

%%%%%%%%%%%%%%%%%%%%%%%%%%%%%%%%%%%%%%%%%%%%%%%%%%%%%%%%%%%%%%%%%%%%%%%%%%%%%%%

\section{Related Work}
\label{sec: related_work}

\textbf{Adversarial Robustness and Interpretability}
A growing body of literature highlights the connection between a model's robustness and the interpretability of its saliency maps. 
Early observations by \citet{tsipras2018robustness} noted a tradeoff between adversarial and standard accuracy, while also revealing that the saliency maps of adversarially trained models (especially with $\ell_\infty$-restricted attacks) align well with perceptually relevant features.
This was further supported by \citet{zhang2019interpreting}, who empirically show that adversarially trained CNNs produce sparser saliency maps that focus more strongly on object shapes and contours. 
Theoretical justifications have been proposed by \citet{etmann19connection,chalasani2020concise} for linear classifiers. 
More recently, analyzing linear classifiers as well, \citet{bhusal2026training} observed that while adversarial training enhances the sparsity and input-stability of saliency maps it can compromise output-side stability.

\textbf{Gradient Regularization} 
\citet{ross2018improving} demonstrated that direct input-gradient regularization not only improves adversarial robustness but also naturally improves model interpretability. 
Building on this, \citet{gong2024structured} proposed replacing fixed-budget adversarial training with a penalty-based approach, utilizing a regularizer to induce desired properties in the gradient with respect to the input.
The TRADES method \citep{zhang2019theoretically} involves regularization with a nonlocal loss-gradient as modification of adversarial training \citep{madry2018toward}.

\textbf{Theoretical Foundations}
\citet{bungert2023geometry} formalized the equivalence of adversarial training to regularization with the nonlocal total variation functional $\TV_\eps$ (see \eqref{eq:Loss+TV}) for the $\ell_1$-loss and gave an existence proof. 
Subsequently, \citet{bungert2024gamma} established Gamma-convergence of this functional. 
\citet{siethoff2025existence} proved existence of solutions for \eqref{eq:Loss+TV} with general loss functions.
Furthermore, \citet{yeh_et_al_2019} and \citet{garcia2022regularized} both provide an upper bound for the adversarial risk in terms of a regularization involving the dual\footnote{dual with respect to the norm used to constrain the adversarial perturbations} norm of the gradient with respect to the input.
Finally, \citet{park2023minimum} leveraged Gamma-convergence to analyze the training with weight decay of two-layer neural networks with ReLU activation, operating in a similar setting to ours.

\section{Sketch of Proof}
\label{sec: proof details}

In this section we provide an overview of the proof of \cref{thm: main}. 
We start by defining 
\begin{align}\label{eq: def F_n}
    F_n(\theta) \coloneq \frac{\widehat{\mathcal{R}}_n(f_\theta)-\inf_\B\mathcal{R}}{\eps_n} + \widehat{\TV}_{\eps_n,n}(f_\theta;\mu) + \frac{\lambda_n}{\eps_n} R_{WD}(\theta)
\, ,
\end{align}
and observe that the minimizers of this functional coincide with the solutions to \eqref{eq: full objective}. 
We divided the objective in \eqref{eq: full objective} by $\eps_n$ for otherwise one would only be able to prove that its minimizers converge to Bayes classifiers, without recovering the selection principle encoded in the functional $F$ from \eqref{eq: def F}.

To prove \cref{thm: main}, we will show that these functionals Gamma-converge \citep{braides2002gamma} to the functional $F$ defined in \eqref{eq:limit-functional} (see \cref{thm: gamma} for the precise statement). 
Gamma-convergence is a notion of convergence that together with a corresponding compactness result---that we provide in \Cref{prop: compactness} for our case---implies that minimizers of the functionals $F_n$ converge to minimizers of the limit functional $F$. 
Formally, establishing Gamma-convergence requires two complementary conditions: the $\liminf$ inequality and the $\limsup$ inequality. 

The $\liminf$ implies that, if we have a sequence of minimizers of $F_n$ that is to known to converge to some limit, this limit has a sufficiently small value of $F$.
\begin{proposition}\label{prop: liminf AT}
    Let \cref{assum: loss,assum: mu} be satisfied.
    For $n\in\N$ let $\theta_n\in\Theta_{m_n}$ such that $f_n := f_{\theta_n}$ converges to a function $f\in L^1(\Omega)$ in $L^1(\Omega)$.
    Moreover, let $(\eps_n)_{n\in\N}$ and $(\lambda_n)_{n\in\N}$ be sequences of positive numbers converging to zero satisfying
    \begin{align}
        \limsup_{n\to\infty}\frac{1}{\eps_n} \sqrt{\frac{\log(n)}{n}} < \infty, \qquad \lim_{n\to\infty} \frac{1}{\lambda_n\sqrt{n}} = 0, \qquad \alpha \coloneq \lim_{n\to\infty} \frac{\lambda_n}{\eps_n} \in (0,\infty). 
    \end{align}
    Then it holds
    \(
        F(f) \leq \liminf_{n\to\infty} F_n(\theta_n)
    \)
    almost surely (with respect to the drawing of samples).
\end{proposition}
The proof of this result, which we carry out in \cref{sec: proof of propositions}, combines Rademacher-based generalization bounds for the empirical risk established in \cref{sec: generalization bounds} with a corresponding inequality for the regularizing term (empirical nonlocal total variation and weight decay) which we prove in \cref{sec: regularizer}. 
This analysis is achieved by combining the properties of the population level functional $\TV_\eps$ (\cref{sec: nonlocal per and tv}) with uniform concentration bounds for its empirical version (\cref{sec: generalization bounds}).

The $\limsup$ inequality complements the $\liminf$ inequality by constructing a sequence that achieves the bound established there.
This will be used to show that the limit of a sequence of minimizers of $F_n$ is actually a minimizer of $F$. 
\begin{proposition}\label{prop: limsup AT}
    Let \cref{assum: mu} be satisfied. 
    Let $f\in L^1(\Omega)$, $(\eps_n)_{n\in\N}$ and $(\lambda_n)_{n\in\N}$ be sequences of positive numbers converging to zero satisfying
    \begin{align}
        \lim_{n\to\infty}\frac{1}{\eps_n} \sqrt{\frac{\log(n)}{n}} = 0, \qquad \lim_{n\to\infty}\frac{1}{\eps_n \sqrt{m_n}} = 0, \qquad \alpha \coloneq\lim_{n\to\infty}\frac{\lambda_n}{\eps_n} \in (0,\infty). 
    \end{align}
    Then, there exists a sequence $(\theta_n)$ with $\theta_n \in\Theta_{m_n}$ such that the sequence $f_n\coloneq f_{\theta_n}$ converges to $f$ in $L^1(\Omega)$ and 
    \(
        \limsup_{n\to\infty} F_n(\theta_n) \leq F(f)
    \)
    almost surely (with respect to the drawing of samples). 
    Such a sequence $(f_n)$ is called recovery sequence for $f$.
\end{proposition}
As for the $\liminf$ inequality, we combine generalization bounds and the results for the population level nonlocal total variation to prove this result (see \cref{sec: proof of propositions}). 
Additionally, since $F(f)=\infty$ otherwise, we can assume that $f$ is a Barron function, and use this to obtain a sequence of neural networks approximating $f$ with dimension-independent rates (\cref{lem:nn_approx_in_l_infty}). 

Finally, as already mentioned, to enable us to extract a convergent subsequence from a sequence of minimizers of $F_n$, we require a compactness result. 
\begin{proposition}\label{prop: compactness}
    Let \Cref{assum: loss} be satisfied and assume that $\alpha \coloneq\lim_{n\to\infty}\lambda_n/\eps_n \in (0,\infty)$.
    For $n\in\N$ let $\theta_n\in\Theta_{m_n}$ such that $\limsup_{n\to\infty} F_n(\theta_n) < \infty$.
    Then, almost surely (with respect to the samples) there exists a subsequence $\theta_{n_j}$ and a function $f\in\B$ such that $f_{\theta_{n_j}}\to f$ in $L^1(\Omega)$. 
\end{proposition}

With these ingredients established, the proof of \cref{thm: main} is a standard argument from the theory of Gamma-convergence. 
For the sake of completeness, we carry out the explicit steps below. 
\begin{proof}[Proof of \cref{thm: main}]
    For $n\in\N$, let $\theta_n$ be a solution to \eqref{eq: full objective} and denote the induced neural network by $f_n = f_{\theta_n}$. 
    As we have mentioned, this implies that $\theta_n$ minimizes $F_n$. 
    By $f^\star$ we denote an approximate minimizer of $F$, i.e. $f^\star\in\B$ with $\mathcal{R}(f^\star) = \inf_\B \mathcal{R}$ such that 
    \begin{align}
        \TV(f^\star;\mu) + \alpha [f^\star]_\B \leq \inf F + \delta
    \end{align}
    for $\delta > 0$. 
    Moreover, let $(f_n^\star)$ be a recovery sequence for $f^\star$ according to \Cref{prop: limsup AT}, i.e.,
    \begin{align}
         \alpha [f^\star]_\B + \TV(f^\star;\mu) = F(f^\star) \geq \limsup_{n\to\infty} F_n(f_n^\star). 
    \end{align}
    almost surely. 
    Since $f_n$ minimizes the functional $F_n$, this implies
    \begin{align}
        \limsup_{n\to\infty} F_n(f_n) \leq \limsup_{n\to\infty} F_n(f_n^\star) \leq \TV(f^\star;\mu) + \alpha[f^\star]_\B < \infty. 
    \end{align}
    Therefore, we can apply \Cref{prop: compactness} and obtain the existence of a function $f\in\B$ and a subsequence $(f_{n_j})$ such that $f_{n_j}\to f$ in $L^1(\Omega)$. 
    We can conclude
    \begin{align}
        F(f) \leq \liminf_{j\to\infty} F_{n_j}(f_{n_j}) \leq \limsup_{j\to\infty} F_{n_j}(f_{n_j}) \leq \TV(f^\star;\mu)+\alpha[f^\star]_\B \leq \inf F + \delta
    \end{align}
    almost surely from \Cref{prop: liminf AT}. 
    Since $\delta > 0$ is arbitrary, this implies that $f$ minimizes $F$. 
\end{proof}

%%%%%%%%%%%%%%%%%%%%%%%%%%%%%%%%%%%%%%%%%%%%%%

\section{Numerical Illustrations}\label{sec:numerical illustrations}

To empirically validate our findings beyond simple two-layer ReLU networks, we compare ResNet18 models trained on clean CIFAR-10 images to those that were adversarially trained on the CIFAR-10 dataset \citep{krizhevsky2009learning}, using different adversarial budgets. 
We report our results in the top row of \cref{fig:sweep}, plotting the mean over 5 random seeds as well extremal values. 
From left to right, we plot the accuracy achieved on clean images, the robust accuracy for images attacked with Projected Gradient Descent (PGD) \citep{madry2018toward} with $50$ steps and a fixed adversarial budget of $8/255$, the $\ell_1$-norms of the gradient averaged over the test set, and the average sparsity of these gradients thresholded at different levels.
Here, sparsity refers to the ratio of components in the gradient whose magnitude lies below the threshold.
%In this context, by gradients with respect to the input, we always mean $\nabla_x f_c(x)$, where $c$ denotes the true label of the data-point $x$. 

We observe the usual trade-off between accuracy and robustness. 
Interestingly, even though training with a small adversarial budget leads to comparable accuracy as standard training, these models exhibit a higher sparsity of their input-gradients and also a significantly lower $\ell_1$-norm, in line with our theoretical results.  
The confidence intervals are hardly visible, indicating high robustness against the stochasticity of neural network training.

In the bottom row of \cref{fig:sweep}, we plot the same metrics for the ResNet50 models\footnote{Taken from \url{https://huggingface.co/madrylab/robust-imagenet-models}} of \citet{salman2020do}, pretrained on the ImageNet dataset \citep{deng2009imagenet}.
Our previous observations remain true, indicating that our results scale to large datasets.
See \cref{sec: additional numerical experiments} for more experimental results.

%%%%%%%%%%%%%%%%%%%%%%%%%%%%%%%%%%%%%
\begin{figure}[htb]
    \centering
    \includegraphics[width=0.24\linewidth,trim=1.1cm 1.1cm 1.1cm 1.1cm,clip]{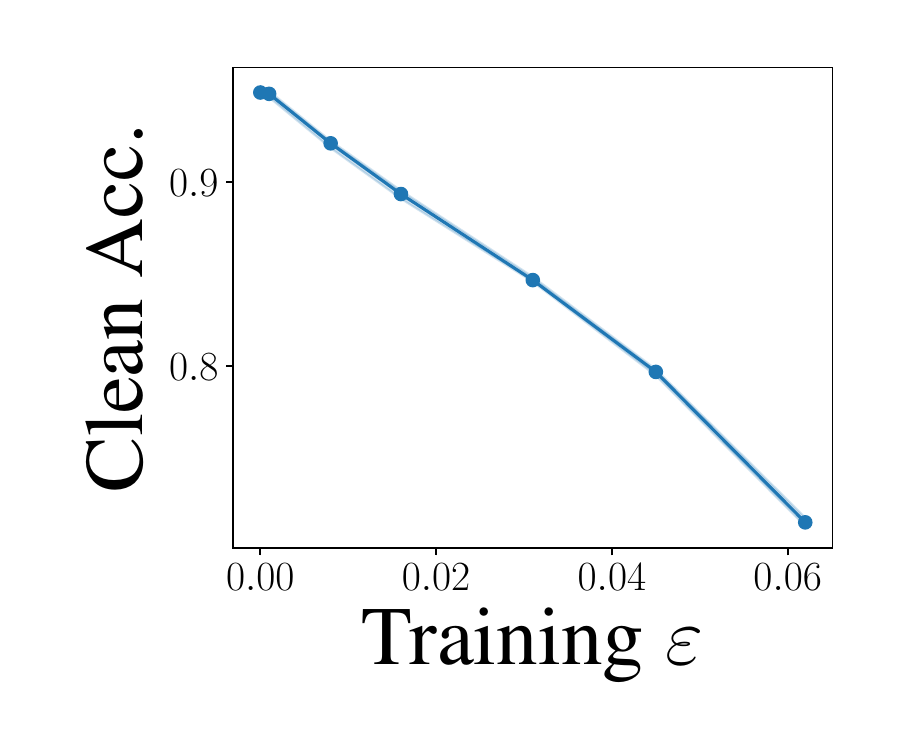}
    \hfill
    \includegraphics[width=0.24\linewidth,trim=1.1cm 1.1cm 1.1cm 1.1cm,clip]{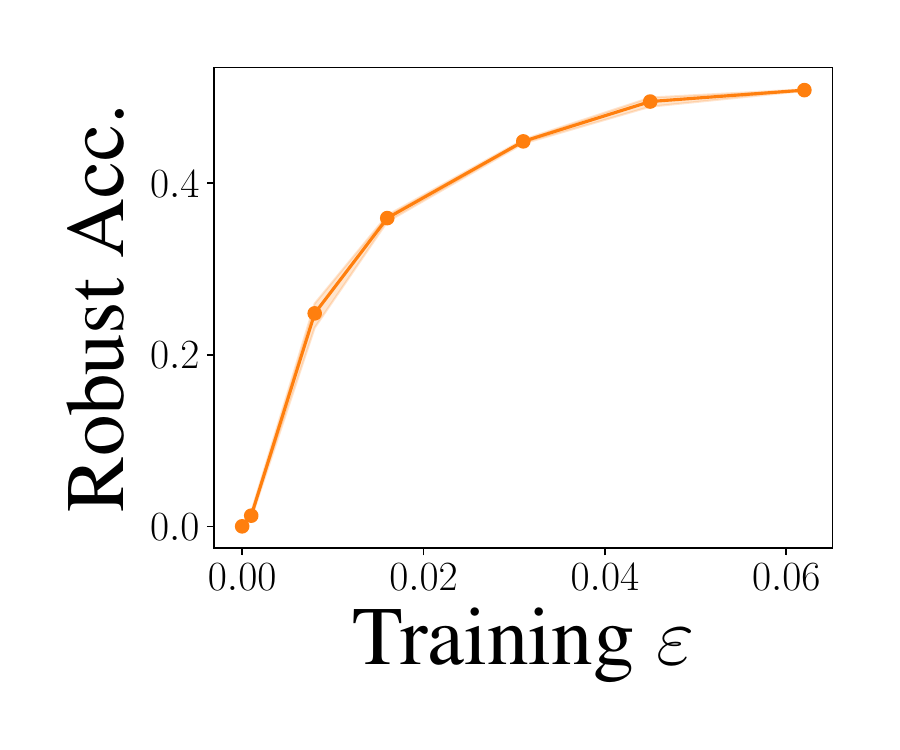}
    \hfill
    \includegraphics[width=0.24\linewidth,trim=1.1cm 1.1cm 1.1cm 1.1cm,clip]{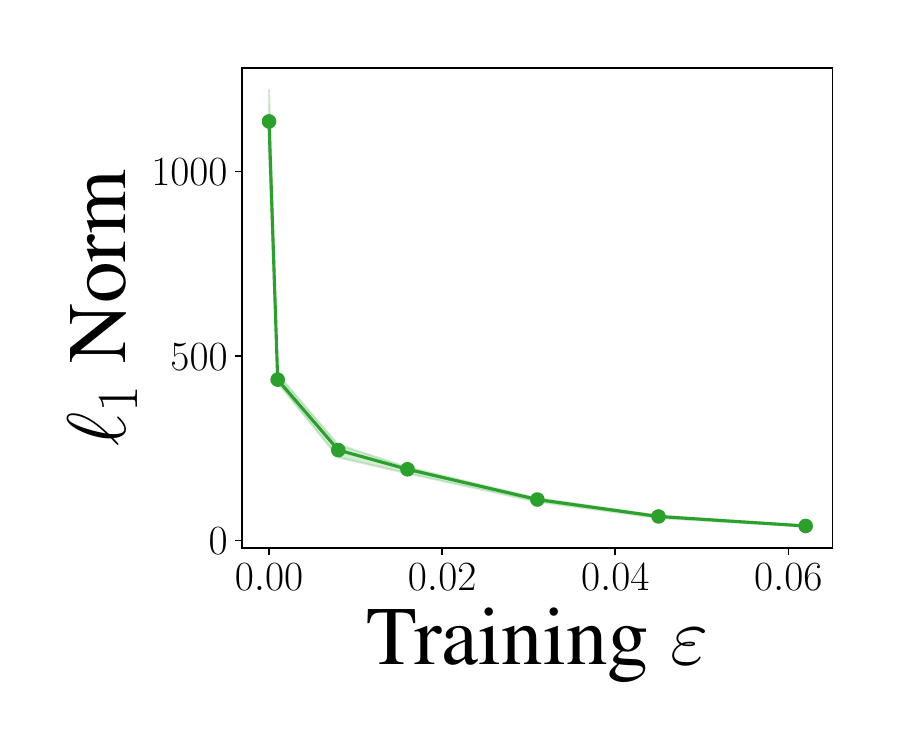}
    \hfill
    \includegraphics[width=0.24\linewidth,trim=1.1cm 1.1cm 1.1cm 1.1cm,clip]{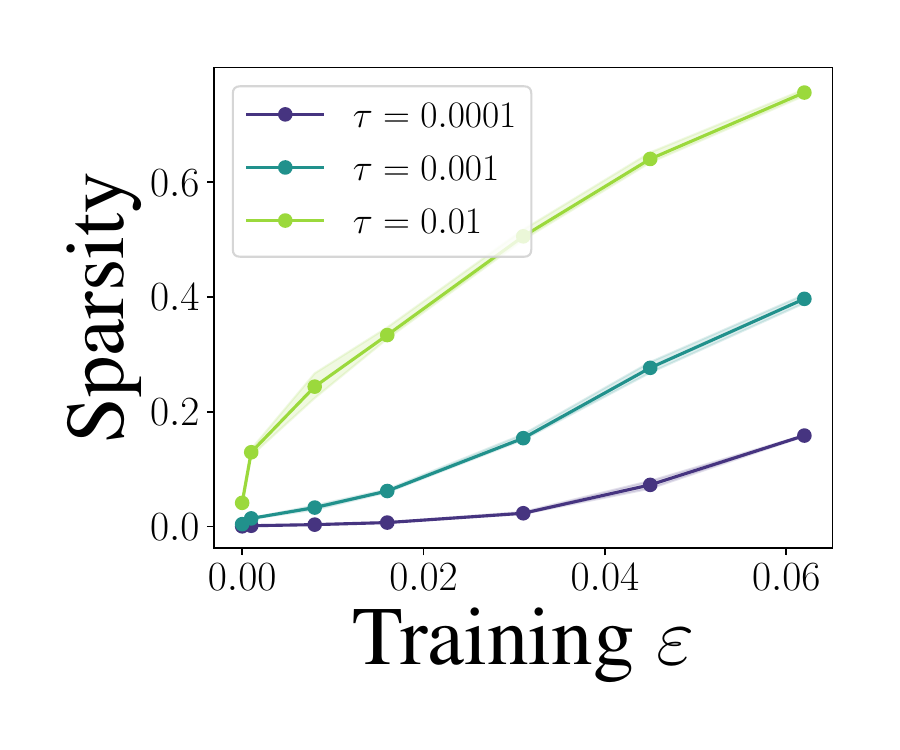}

    \includegraphics[width=0.24\linewidth,trim=1.1cm 1.1cm 1.1cm 1.1cm,clip]{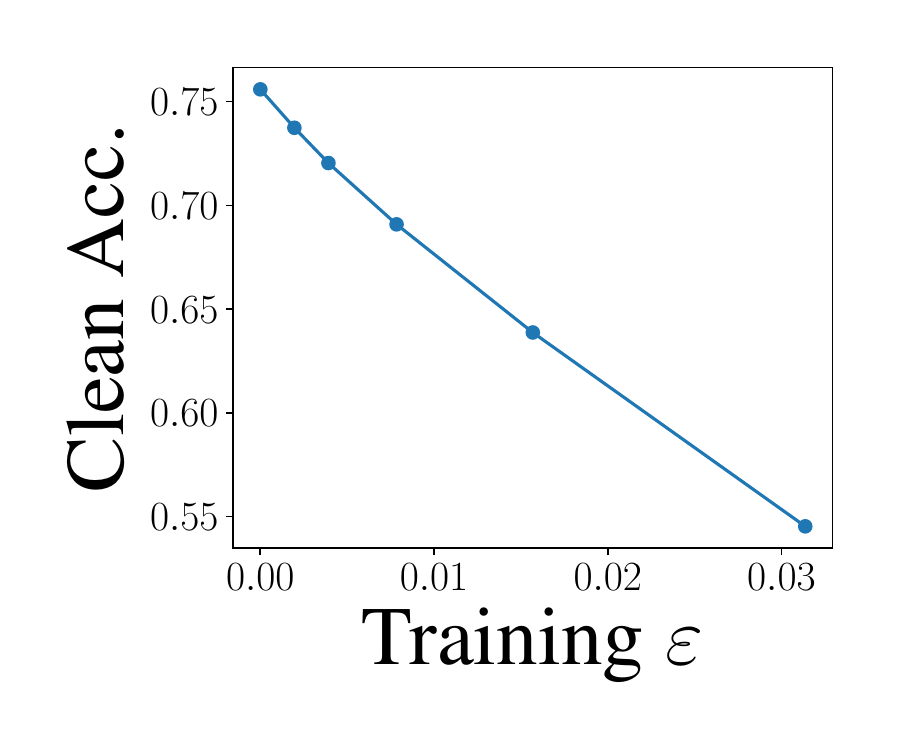}
    \hfill
    \includegraphics[width=0.24\linewidth,trim=1.1cm 1.1cm 1.1cm 1.1cm,clip]{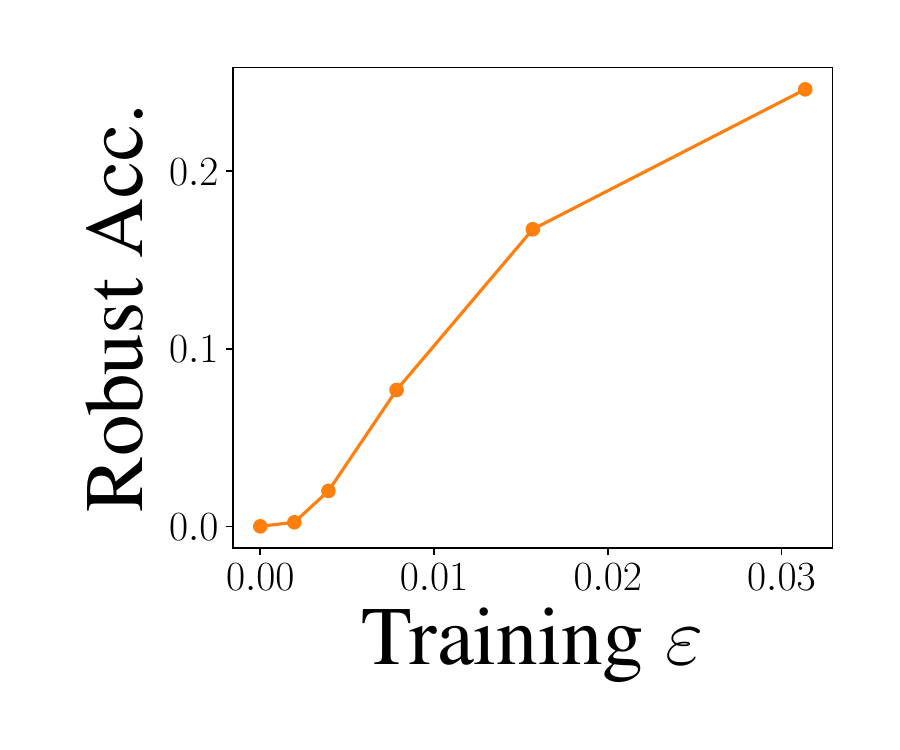}
    \hfill
    \includegraphics[width=0.24\linewidth,trim=1.1cm 1.1cm 1.1cm 1.1cm,clip]{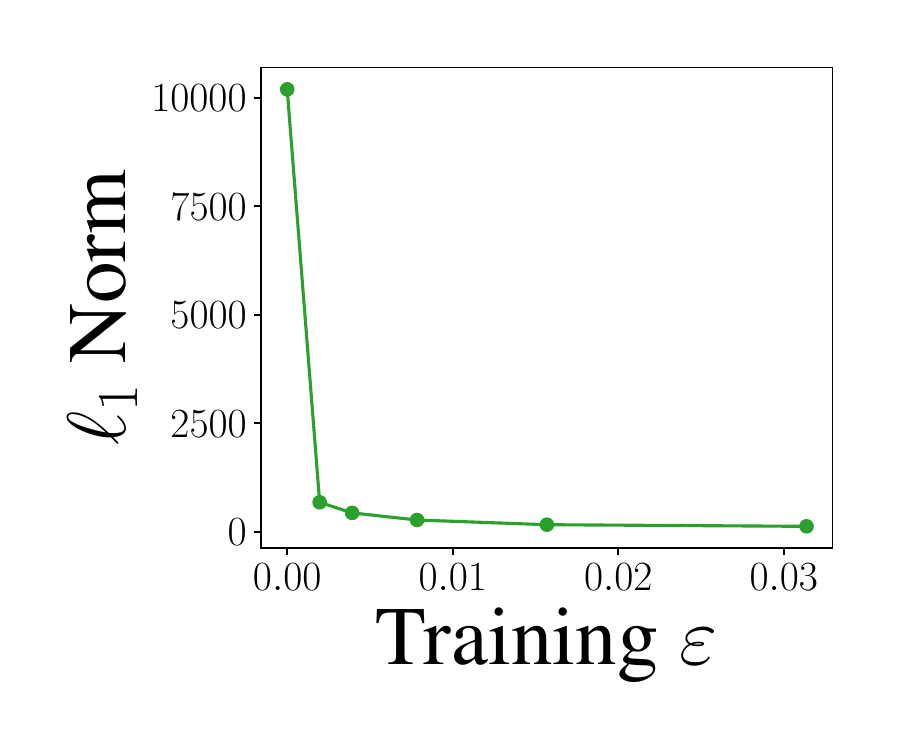}
    \hfill
    \includegraphics[width=0.24\linewidth,trim=0.8cm 1.1cm 1.1cm 1.1cm,clip]{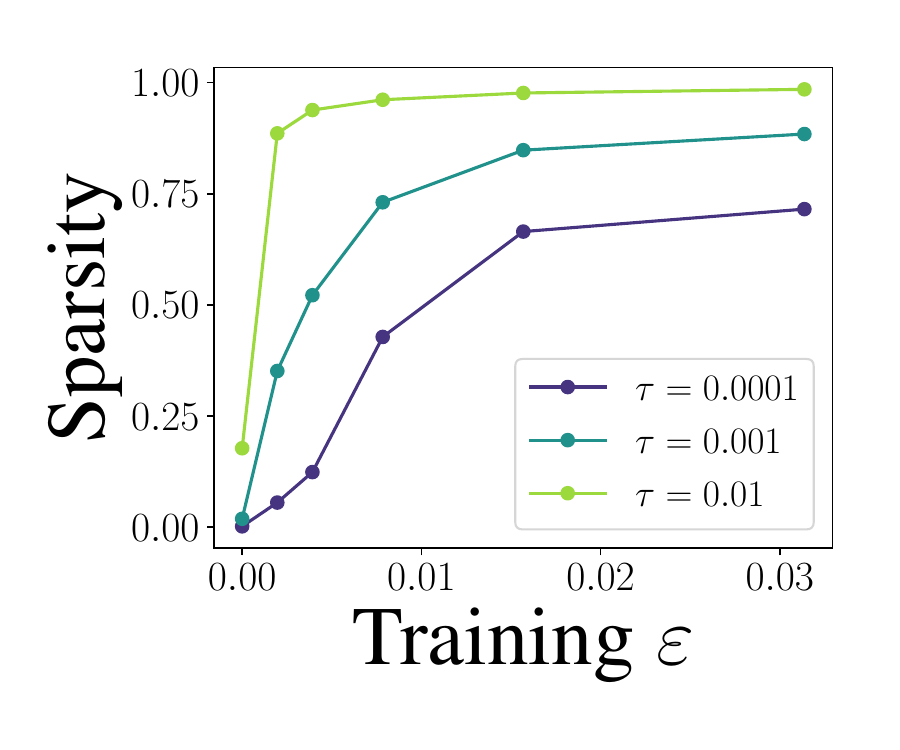}
    \caption{Accuracy and gradient metrics for ResNet18 on CIFAR-10 (top) and ResNet50 on ImageNet (bottom). From left to right, we display clean accuracy, robust accuracy, the average $\ell_1$-norm, and the average sparsity (thresholded at different values) of the input gradients.}
    \label{fig:sweep}
\end{figure}

%%%%%%%%%%%%%%%%%%%%%%%%%%%%%%%%%%%%%%%%%%%%%%

\section{Conclusion}

We provided a theoretical explanation for the empirical observation that gradient saliency maps of adversarially-trained neural networks tend to be sparser and more perceptually aligned than those of naturally-trained ones.
For this we analyzed a total variation and weight decay regularized empirical risk minimization problem over the class of two-layer neural networks with ReLU activation, and proved that it selects a subclass of Bayes classifiers, consisting of those which minimize a balance of their average gradient $\ell_1$-norm and their Barron norm.  
Hence, classifiers with axis-aligned gradients are favored, which explains the sparsity of their gradient maps.

Our proofs involve generalization bounds, Monte Carlo approximation arguments, and Gamma-convergence techniques, and our numerical experiments confirm the theoretical results and indicate their generalizability.
Our current analysis requires the weight decay regularization to be equally strong as the adversarial one and does not cover the original form of adversarial training, proposed by \cite{madry2018toward}.
Future work includes investigating the connection between TV regularization and adversarial training theoretically and empirically as well as studying their influence on other saliency maps like integrated or smoothed gradients.

% DGA: [ideas for future work, here for safe-keeping]
% - can we say something about other saliency maps (eg integrated gradient)? Yannick has preliminary results, weird spike for smoothgrad near eps = 0
% - other limit regimes (e.g. NTK)
% - saliency maps for LLMs + adversarial training -> same empirical observation? 

%%%%%%%%%%%%%%%%%%%%%%%%%%%%%%%%%%%%%%%%%%%%%%%%%%%%%%%%%%%

\subsection*{AI use statement}

In this work we used generative AI tools for polishing the manuscript, discussing parts of the mathematical proofs and assisting with code to generate the numerical experiments.
We have not used AI tools for generating the core research hypothesis or conducting literature reviews.
We have reviewed all AI assisted work.
Specifically, we verified all AI-discussed proof steps and tested the generated code for correctness.
We take responsibility for the final content of this work,
including text, claims or artifacts produced with the aid of generative AI.

\subsubsection*{Acknowledgments}
YL and LB acknowledge funding by the Federal Ministry of Research, Technology and Space of Germany (BMFTR) under grant agreement No. 16IS24072 (COMFORT).
LB acknowledges funding by the Deutsche Forschungsgemeinschaft (DFG, German Research Foundation) – project number 544579844 (GeoMAR).

\bibliography{./bibliography}
\bibliographystyle{iclr2027_conference}

\appendix

\section*{Supplementary material}

The supplementary material for our paper is organized into the following appendices. 
In particular, for readers familiar with Gamma-convergence we would like to point our \cref{thm: gamma} in \cref{sec: proof of propositions} which, together with \cref{prop: compactness}, implies our main result \cref{thm: main}. 
However, for convenience, the proof of \cref{thm: main} given in \cref{sec: proof details} of the main paper is a direct one. 

\startcontents[main]
\printcontents[main]{}{1}{}

%%%%%%%%%%%%%%%%%%%%%%%%%%%%%%%%%%%%%%%%%%%%%%%%%%%%%%%%%%

% In this appendix, $C$ denotes a numerical constant whose value can change from one result to the next. 

\section{The Homogeneous Barron Space}
\label{sec: barron space}

Recall our definition of the homogeneous Barron space $\B(\Omega)$ as all functions of the form
\begin{align}
    f_{\pi,y}(x) \coloneq y + \int_{\R\times\R^d\times\R}a\big(\sigma(w\cdot x+b)-\sigma(b)\big)\de\pi(a,w,b),\qquad x\in\R^d,
\end{align}
such that $[f]_{\B(\Omega)}<\infty$, where $[f]_{\B(\Omega)} \coloneq \inf_{f_{\pi,y} = f \text{ on }\Omega} R_{WD}(\pi)$ and 
\begin{align}
    R_{WD}(\pi)\coloneq \frac{1}{2}\int_{\R\times\R^d\times\R}\big(|a|^2+\|w\|_2^2\big)\de\pi(a,w,b). 
\end{align}
Furthermore, we recall that in our setting $\Omega \subset \R^d$ is assumed to be a bounded domain with Lipschitz boundary and that $\mu$ is a probability measure on $\Omega\times\{0,1\}$ describing the distribution of the training data.
Note that the slightly non-standard representation of $f_{\pi,y}$ (see also \eqref{eq: homogeneous barron}) is used to make sure that $R_{WD}(\pi)<\infty$ implies that the integral in \eqref{eq: homogeneous barron} exists, see \citet{park2023minimum} for details.

One can easily see that every function in this space is Lipschitz continuous and the Lipschitz constant is bounded by the Barron semi-norm. 

\begin{lemma}\label{lem:lipschitz_barron}
   Let $f\in\B(\Omega)$. Then, 
   \begin{align}
       |f(x)-f(\tilde{x})| \leq [f]_\B \|x-\tilde{x}\|_2
   \end{align}
   for $x,\tilde{x}\in\Omega$, i.e. $f$ is Lipschitz continuous. 
\end{lemma}

\begin{proof}
    Let $x,\tilde{x}\in\Omega$ and let $\pi, y$ such that $f=f_{\pi,y}$. Then, since the ReLU activation function is $1$-Lipschitz, 
    \begin{align}
        |f(x)-f(\tilde{x})| &= \left| \int a(\sigma(w\cdot x+b) - \sigma(w\cdot\tilde{x}+b)) \de \pi(a,w,b)\right| \\
        &\leq \int |a||w\cdot x - w \cdot\tilde{x}| \de\pi(a,w,b) \\
        &\leq\int|a|\|w\|_2\de\pi(a,w,b) \|x-\tilde{x}\|_2 \\
        &\leq  \|x-\tilde{x}\|_2 R_{WD}(\pi).
    \end{align}
    Taking the infimum over all representations $f_{\pi,y}$ yields the desired estimate. 
\end{proof}
Since $\Omega$ is assumed to be bounded, the Lipschitz continuity of Barron functions implies the inclusion $\B(\Omega) \subset L^\infty(\Omega) \subset L^1(\Omega)$. 

We note that the above proof holds for any Lipschitz continuous activation function $\sigma$ with the only modification being that the Lipschitz constant becomes $\Lip(\sigma)[f]_{\B}$. 
In contrast, our proof of the following result--which established lower semi-continuity of the Barron semi-norm with respect to $L^1$ convergence--relies on the specific structure of the ReLU function. 

\begin{lemma}\label{lem: semi norm lower semicontinuous}
    Let $(f_n)\subset\B(\Omega)$ and $f\in L^1(\Omega)$ such that $f_n \to f$ in $L^1(\Omega)$. Then, 
    \begin{align}
        [f]_\B \leq \liminf_{n\to\infty} [f_n]_\B.
    \end{align}
\end{lemma}

\begin{proof}
    By passing to a subsequence, we can assume that the liminf is a limit. 
    Moreover, if this limit is infinite, the claimed inequality is trivially satisfied.
    Hence, we consider the case
    \begin{align}
        \lim_{n\to\infty}[f_n]_\B =: M < \infty.
    \end{align}
    By definition of the Barron-semi norm, for any $n\in\N$, we can find a probability measure $\pi_n$ over $\R\times\R^d\times\R$ and a real number $y_n$ such that $f_n = f_{\pi_n,y_n}$ and $R_{WD}(\pi_n) < [f_n]_\B + \frac{1}{n}$.
    In particular, $R_{WD}(\pi_n)$ is uniformly bounded by some constant $K$.
    Moreover, defining the signed measure $\tilde{\pi}_n$ by $\de\tilde\pi_n (a,w,b) = a\|w\|_2\de\pi_n(a,w,b)$ leads to 
    \begin{align}
        \|\tilde\pi_n\|_{\TV} = \int_{\R^{d+2}} |a|\|w\|_2 \de\pi_n(a,w,b) \leq R_{WD}(\pi_n) \leq K,
    \end{align}
    and using the fact that the ReLU functions is positively 1-homogeneous, we can compute
    \begin{align}
        f_n(x) &= y_n + \int_{\R^{d+2}} a(\sigma(w\cdot x +b)-\sigma(b))\de\pi_n(a,w,b) \\
        &= y_n + \int_{\R^{d+2}} a\|w\|_2 \left(\sigma\left(\frac{w}{\|w\|_2}\cdot x + \frac{b}{\|w\|_2}\right) - \sigma\left(\frac{b}{\|w\|_2}\right)\right) \de\pi_n(a,w,b) \\
        &= y_n + \int_{\R^{d+2}} \left(\sigma\left(\frac{w}{\|w\|_2}\cdot x + \frac{b}{\|w\|_2}\right) - \sigma\left(\frac{b}{\|w\|_2}\right)\right) \de\tilde\pi_n(a,w,b) \\
        &= y_n + \int_{\R^{d+1}}(\sigma(\hat{w}\cdot x +\hat{b})-\sigma(\hat{b})) \de (T_\sharp \tilde\pi_n)(\hat{w}, \hat{b}), 
    \end{align}
    where $T(a,w,b)\coloneq (\frac{w}{\|w\|_2}, \frac{b}{\|w\|_2})$ for $w\neq 0$ and $T(a,w,b) = 0$ otherwise. 
    Note that the signed measure $\nu_n\coloneq T_\sharp \tilde\pi_n$ is supported on $\mathbb{S}^{d-1}\times \R$ and $\|\nu_n\|_{\TV} \leq \|\tilde\pi_n\|_{\TV} \leq K$. 
    We additionally want to modify the measure $\nu_n$ such that it also has compact support in $\hat{b}$.
    To do so, we make use of the structure of the ReLU activation function and that the domain $\Omega$ is bounded. 
    Let $R\coloneq \sup_{x\in\Omega}\|x\|_2$ and define $\phi_x(w,b) \coloneq \sigma(w\cdot x+b)-\sigma(b)$. 
    We observe that for $|b|>R, w\in\mathbb{S}^{d-1}$ and $x\in\Omega$, we have $\phi_x(w,b) = \phi_x(w,\sign(b)R)$. 
    Consequently, defining
    \begin{align}
        P(w,b) \coloneq \begin{dcases}
            (w,R), \quad &b > R, \\
            (w,b), \quad &-R \leq b \leq R,\\
            (w,-R), \quad &b < -R
        \end{dcases}
    \end{align}
    and $\mu_n \coloneq P_\sharp \nu_n$, we obtain 
    \begin{align}
        f_n(x) =y_n +  \int_{\mathbb{S}^{d-1}\times[-R,R]} (\sigma(\hat{w}\cdot x + \hat{b})-\sigma(\hat{b}))\de \mu_n(\hat{w}, \hat{b})
    \end{align}
    and $\|\mu_n\|_{\TV} \leq \|\nu_n\|_{\TV}\leq \R_{WD}(\pi_n) \leq K$. 
    Note that we only need to integrate over $\mathbb{S}^{d-1}\times[-R,R]$ here, since the support of $\mu_n$ is contained in this compact set. 
    Hence, by \citet[Theorem 1.59]{ambrosio2000functions}, there exists a subsequence, that we again denote by $(\mu_n)$ that weakly* converges to a signed measure $\mu$ and 
    \begin{align}
        \|\mu\|_{\TV} \leq \liminf_{n\to\infty} \|\mu_n\|_{\TV} \leq \liminf_{n\to\infty} R_{WD}(\pi_n) = \lim_{n\to\infty}[f_n]_\B = M.
    \end{align}
    We can also show that the sequence of numbers $(y_n)$ is bounded and thus admits a converging subsequence, since
    \begin{align}
        |y_n|\L^d(\Omega) &= \int_\Omega |y_n|\de x \leq \int_\Omega |f_n(x)|\de x + \int_\Omega \int_{\R^{d+2}} |a|\|w\|_2\|x\|_2 \de \pi_n(a,w,b) \de x \\
        &\leq \|f_n\|_{L^1(\Omega)} + R\L^d(\Omega) K,
    \end{align}
    where $\L^d$ denotes the $d$-dimensional Lebesgue measure. 
    The right-hand side is uniformly bounded since $(f_n)$ converges in $L^1(\Omega)$. 
    Therefore in the following, we assume that $y_n$ converges to some $y\in\R$.
    Using the fact that $\phi_x$ is a test function for the weak* convergence of $\mu_n$ ($\phi_x$ is continuous and $\mathbb{S}^{d-1}\times[-R,R]$ is compact), we observe
    \begin{align}
        \lim_{n\to\infty} f_n(x) &= \lim_{n\to\infty} \left( y_n + \int_{\mathbb{S}^{d-1}\times[-R,R]} (\sigma(\hat{w}\cdot x + \hat{b})-\sigma(\hat{b})) \de\mu_n(\hat{w},\hat{b})\right) \\
        &= y + \int_{\mathbb{S}^{d-1}\times[-R,R]} (\sigma(\hat{w}\cdot x + \hat{b})-\sigma(\hat{b})) \de\mu(\hat{w},\hat{b})
    \end{align}
    and since $f_n\to f$ in $L^1(\Omega)$, the right hand side must coincide with $f$ almost everywhere on $\Omega$. 
    It remains to rewrite the latter integral as an integral with respect to some probability distribution $\pi$ over $\R\times\R^{d}\times\R$.
    To do so, let $\mu = \mu^+-\mu^-$ denote the Hahn-Jordan decomposition of $\mu$ and let $C\coloneq \|\mu\|_{\TV} = \mu^+(\R^{d+1}) + \mu^-(\R^{d+1})$. 
    Without loss of generality, we can assume that $C > 0$, since in the case $C = 0$, we have $\mu = 0$ and can define $\pi \coloneq \delta_{0}$ with $R_{WD}(\delta_0) = 0$. 
    For $C > 0$, we define $S^\pm (\hat{w},\hat{b}) \coloneq (\pm\sqrt{C}, \sqrt{C}\hat{w}, \sqrt{C}\hat{b})$.
    Then, $\pi\coloneq \frac{1}{C}(S^+_\sharp \mu^+ + S^-_\sharp\mu^-)$ defines a probability measure on $\R\times\R^d\times\R$ and 
    \begin{align}
        &\int_{\R^{d+2}} a \phi_x(w,b) \de\pi(a,w,b)\\
        &= \frac{1}{C} \int_{\R^{d+2}} a \phi_x(w,b) \de (S^+_\sharp \mu^+)(a,w,b) - \frac{1}{C} \int_{\R^{d+2}} a \phi_x(w,b) \de (S^-_\sharp \mu^-)(a,w,b) \\
        &= \frac{1}{C}\int_{\R^{d+1}} \sqrt{C}\phi_x(\sqrt{C}\hat{w}, \sqrt{C}\hat{b}) \de \mu^+(\hat{w},\hat{b}) - \frac{1}{C}\int_{\R^{d+1}} \sqrt{C}\phi_x(\sqrt{C}\hat{w}, \sqrt{C}\hat{b}) \de \mu^-(\hat{w},\hat{b}) \\
        &= \int_{\R^{d+1}} \phi_x(\hat{w},\hat{b}) \de(\mu^+-\mu^-)(\hat{w}, \hat{b}) \\
        &= \int_{\R^{d+1}} \phi_x(\hat{w},\hat{b}) \de\mu(\hat{w},\hat{b}). 
    \end{align}
    due to the positive $1$-homogeneity of $\sigma$. 
    Hence, we can conclude $f=f_{\pi,y}$ and $[f]_\B \leq R_{WD}(\pi) = \|\mu\|_{\TV} \leq \liminf_{n\to\infty}[f_n]_\B$.
\end{proof}

A fundamental property of Barron functions is that they can be approximated by neural networks with dimension-independent rates. 
For our application it is particularly useful to establish approximation in the $L^\infty$-norm. 

\begin{lemma}\label{lem:nn_approx_in_l_infty}
Let $f\in\B$. Then, for every $m\in\N$, there exists a neural network $f_m$ of the form \eqref{eq:neural network} parametrized by $\theta_m$ such that 
    \begin{align}
        \|f-f_m\|_{L^\infty(\Omega)} \leq \frac{C[f]_\B}{\sqrt{m}}
    \end{align}
    with some positive constant $C$ and
    \begin{align}
        R_{WD}(\theta_m) \leq [f]_\B.
    \end{align}
\end{lemma}

 The proof of this statement goes along the lines of \citet[Theorem 12]{ma2020towards}, who consider the case $\Omega=[0,1]^d$ and non-homogeneous Barron spaces.
 For the sake of completeness, we adapt their proof to our setting.
 
\begin{proof}
Let $\pi$ be a probability measure on $\R\times\R^d\times\R$ and $y\in\R$ such that $f=f_{y,\pi}$ and $R_{WD}(\pi) \leq [f]_\B + \delta$ for arbitrary positive $\delta$.
    We want to normalize $\pi$ in a way such that 
    \begin{align}
        |a|=\|w\|_2\leq R_{WD}(\pi)
    \end{align}
    $\pi$-almost surely, which will happen in a two step procedure.
    First, we make sure that the product $|a|\|w\|_2$ is constant $\pi$-almost surely.
    To do so, let $C \coloneq  \int_{\R^{d+2}} |a|\|w\|_2 \de\pi(a,w,b)$ and define the probability measure $\tilde{\pi}$ via $\de\tilde\pi = \frac{|a|\|w\|_2}{C}\de\pi$. 
    We may assume without loss of generality that $C > 0$.
    The case $C=0$ implies that $f$ is a constant function, which can trivially be parametrized by a neural network of the form \eqref{eq:neural network}. 
    Defining $S(a,w,b) \coloneq (\frac{aC}{|a|\|w\|_2}, w, b)$, and considering the push-forward $\bar{\pi} = S_\sharp\tilde{\pi}$, we observe that for $h_{(a,w,b)}(x) \coloneq a(\sigma(w\cdot x + b) - \sigma(b))$, we have
    \begin{align}
        \int_{\R^{d+2}} h_{(a,w,b)}(x) \de\bar{\pi}(a,w,b) &= \int_{\R^{d+2}} h_{S(a,w,b)}(x) \de\tilde{\pi}(a,w,b) \\
        &= \int_{\R^{d+2}} \frac{aC}{|a|\|w\|_2}(\sigma(w\cdot x+b)-\sigma(b)) \frac{|a|\|w\|_2}{C} \de\pi(a,w,b) \\
        &= f(x)-f(0).
    \end{align}
    Hence, $\bar{\pi}$ represents $f$ and we have $|a|\|w\|_2 = C$ $\bar{\pi}$-almost surely.
    Defining 
    \begin{align}
        T(a,w,b) \coloneq \left( a\sqrt{\frac{\|w\|_2}{|a|}}, w\sqrt{\frac{|a|}{\|w\|_2}}, b\sqrt{\frac{|a|}{\|w\|_2}}\right),
    \end{align}
    one can show using the $1$-homogeneity of the ReLU activation function that $\hat{\pi} \coloneq T_\sharp \bar\pi$ also represents $f$ and additionally satisfies that $|a|^2=\|w\|_2^2 = C \leq R_{WD}(\pi)$ almost surely. 
    As in the proof of \Cref{lem: semi norm lower semicontinuous}, we can additionally construct the measure $\hat{\pi}$ such that its support w.r.t. the last variable is contained in some interval $[-K, K]$, where $K$ depends on the domain $\Omega$ and on $[f]_\B$. 
    For simplicity, we will from now on refer to $\hat{\pi}$ as $\pi$.
    
    Using \citet[Lemma 26.2]{shalev-shwartz2014understanding} yields
    \begin{align}
        &\IE_{(a,w,b)\sim\pi^m}\left[\sup_{x\in\Omega} \left(f(x)-y-\frac{1}{m}\sum_{i=1}^m a_i(\sigma(w_i\cdot x + b_i)-\sigma(b_i)\right)\right] \\
        &\leq 2\IE_{(a,w,b)\sim\pi^m} \IE_\xi \left[\sup_{x\in\Omega}\frac{1}{m}\sum_{i=1}^m \xi_ia_i(\sigma(w_i\cdot x +b_i)-\sigma(b_i))\right],
    \end{align}
    where $\xi\in\R^m$ denotes a random vector whose components are independent Rademacher random variables. 
    Moreover, 
    \begin{align}
        \IE_\xi \left[\sup_{x\in\Omega}\frac{1}{m}\sum_{i=1}^m \xi_ia_i\left(\sigma(w_i\cdot x +b_i)-\sigma(b_i)\right) \right] &\leq \IE_\xi\left[\sup_{x\in\Omega}\frac{1}{m}\sum_{i=1}^m \xi_i|a_i||w_i\cdot x|\right] \\
        &= \frac{1}{m}\IE_\xi\left[\sup_{x\in\Omega} x \cdot\sum_{i=1}^m \xi_i a_i w_i\right] \\
        &\leq \frac{R}{m}\IE_\xi\left[\left\lVert \sum_{i=1}^m \xi_i a_iw_i\right\rVert_2\right],
    \end{align}
    where $R\coloneq \sup_{x\in\Omega}\|x\|_2$. 
    Defining $z_i \coloneq a_iw_i \in\R^d$ yields $\|z_i\|_2 \leq |a_i|\|w_i\|_2 \leq [f]_\B+\delta$. 
    Hence, we can further estimate
    \begin{align}
        \frac{R}{m}\IE_\xi\left[\left\lVert \sum_{i=1}^m \xi_i a_iw_i\right\rVert_2\right] &\leq \frac{R}{m} \sup_{\|z_i\|_2\leq [f]_\B+\delta}\IE_\xi \left[\left\lVert \sum_{i=1}^m \xi_iz_i\right\rVert_2\right] \\
        &\leq \frac{R([f]_\B+\delta)}{m}\sup_{\|z_i\|_2\leq 1} \IE_\xi \left[\left\lVert \sum_{i=1}^m \xi_iz_i\right\rVert_2\right] \\
        &\leq \frac{R([f]_\B+\delta)}{\sqrt{m}}. 
    \end{align}
    In the last step, we used Jensen's inequality together with the fact that the $\xi_i$ are independent with expected value zero to bound the expectation by $\sqrt{m}$. 
    Overall, we obtain
    \begin{align}
        \IE_{(a,w,b)\sim\pi^m}\left[\sup_{x\in\Omega} \left(f(x)-y-\frac{1}{m}\sum_{i=1}^m a_i(\sigma(w_i\cdot x + b_i)-\sigma(b_i)\right)\right] \leq \frac{2R([f]_\B+\delta)}{\sqrt{m}}. 
    \end{align}
    This bound for the expectation means that there must exist a choice of parameters $(a_i,w_i,b_i)$ for $i=1,\dots,m$ such that $|a_i|^2=\|w_i\|_2^2 \leq [f]_\B + \delta$ such that
    \begin{align}
        \left\| f - b_0 - \frac{1}{m}\sum_{i=1}^m a_i\sigma(w_i\cdot x + b_i)\right\|_{L^\infty(\Omega)} \leq \frac{C([f]_\B+\delta)}{\sqrt{m}}
    \end{align}
    with $b_0 = y + \sum_{i=1}^m \sigma(b_i)$. 
    To keep notation aligned with \eqref{eq:neural network}, the parameters of the neural network are given by $(\frac{a_i}{\sqrt{m}}, \frac{w_i}{\sqrt{m}}, \frac{b_i}{\sqrt{m}})$ and the resulting parametrization $\theta_m$ satisfies 
    \begin{align}
        R_{WD}(\theta_m) = \frac{1}{2}\sum_{I=1}^m \left(\frac{|a_i|^2}{m} + \frac{\|w_i\|_2^2}{m}\right) \leq [f]_\B + \delta.
    \end{align}
    Since the parameters lie in a compact subset of $\R^{d+2}$, when letting $\delta\to 0$, we can extract a convergent subsequence. 
    The limit satisfied the above inequalities without $\delta$, since the map $\R^{d+2}\ni(a,w,b)\mapsto a\sigma (w^T\cdot + b) \in L^\infty(\Omega)$ is continuous. 
\end{proof}

For proving generalization bounds (see \cref{sec: generalization bounds}), we must investigate the Rademacher complexity of the homogeneous Barron space. 
Specifically, we consider the subset of Barron functions with a uniform bound for the Barron semi-norm and for their value at zero, 
\begin{align}\label{eq: def FQB}
    \mathcal{F}_{Q, B} \coloneq \{f\in\B\,|\, [f]_\B \leq Q, |f(0)| \leq B\}.
\end{align}
Similar to \citet[Corollary D.3]{park2023minimum}, we obtain the following estimate for the Rademacher complexity of this set. 
\begin{lemma}\label{lem:Rademacher bound}
Let $Q,B \geq 0$. 
Then, there exists a constant $C$ depending on $\Omega$ such that for all probability measures $\rho$ that are supported on $\Omega$, we have
    \begin{align}
        \Rad_{\rho,n}(\mathcal{F}_{Q,B}) &\leq \frac{CQ}{\sqrt{n}} + \frac{B}{\sqrt{n}}.
    \end{align}
\end{lemma}
\begin{proof}
    For a fixed sample $S=\{x_i, \dots, x_n\}$ and independent Rademacher random variables $\xi_1,\dots,\xi_n$, we have
    \begin{align}
        \Rad_S(\mathcal{F}_{Q,B}) &= \IE_\xi\left[\sup_{\substack{f\in\mathcal{F}_{Q,0},\\c\in[-B,B]}} \frac{1}{n}\sum_{i=1}^n \xi_i(f(x_i)+c)\right] \\
        &\leq \IE_\xi\left[\sup_{f\in\mathcal{F}_{Q,0}} \frac{1}{n}\sum_{i=1}^n \xi_if(x_i)\right] + \IE_\xi\left[\sup_{c\in[-B,B]}\frac{1}{n}c\sum_{i=1}^n \xi_i\right] \\
        &\leq \Rad_S(\mathcal{F}_{Q,0}) + \IE_\xi\left[B\left\lvert \frac{1}{n}\sum_{i=1}^n \xi_i\right\rvert \right] \\
        &\leq \Rad_S(\mathcal{F}_{Q,0}) + \frac{B}{\sqrt{n}}. 
    \end{align}
    In the last step, we have used that by Jensen's inequality, 
    \begin{align}
        \IE_\xi\left[\left|\sum_{i=1}^n \xi_i\right|\right] \leq \sqrt{\IE_\xi\left[\left|\sum_{i=1}^n \xi_i\right|\right]^2} = \sqrt{\IE_\xi\left[\sum_{i=1}^n \xi_i^2\right]} = \sqrt{n}. 
    \end{align}
    Hence, taking the expectation over the data samples, we obtain
    \begin{align}
        \Rad_{\rho,n}(\mathcal{F}_{Q,B}) &\leq \Rad_{\rho,n}(\mathcal{F}_{Q,0}) + \frac{B}{\sqrt{n}} \\
        &\leq \frac{CQ}{\sqrt{n}} + \frac{B}{\sqrt{n}}. 
    \end{align}
    In the last step, we used the bound $\Rad_{\rho,n}(\mathcal{F}_{Q,0}) \leq \frac{CQ}{\sqrt{n}}$ which holds for data distributions $\rho$ supported on $\Omega$, where the constant $C$ depends only on $\Omega$.
    This bound follows readily by adapting the proof of \citet[Lemma D.2]{park2023minimum}. 
\end{proof}

For our theory it is also important to obtain generalization bounds for the nonlocal total variation $\TV_\eps$. 
Due to its structure, we estimate the Rademacher complexity of the following class of functions
\begin{align}\label{eq:def GQeps}
    \mathcal{G}_{Q,\eps} &\coloneq \{\Omega\ni x \mapsto \sup_{B_\eps(x)\cap\Omega} f - f(x) \,|\, f\in \mathcal{F}_{Q,\infty}\} \\
    &= \{\Omega\ni x \mapsto \sup_{B_\eps(x)\cap\Omega} f - f(x) \,|\, f\in \mathcal{F}_{Q,0}\}
\end{align}
for $Q,\eps > 0$, where 
\begin{align}
    \mathcal{F}_{Q,\infty} \coloneq \{f\in\B(\Omega) \,|\, [f]_{\B(\Omega)} \leq Q\}.
\end{align}

\begin{lemma}\label{lem: Rademacher bound 2}
There exists a constant $C > 0$ depending on $\Omega$, such that for all $Q,B > 0$ and $\eps \in(0,1]$ and probability measures $\rho$ supported on $\Omega$, we have
    \begin{align}
        \Rad_{\rho,n}(\mathcal{G}_{Q, \eps}) \leq \frac{CQ}{\sqrt{n}}. 
    \end{align}
\end{lemma}
To prove this statement, we first reduce the problem to functions represented by one neuron and exploit the fact that the supremum over a given ball can be computed explicitly for functions of this form. 
From there, we proceed exactly as in the proof of \citet[Lemma D.2]{park2023minimum}.
\begin{proof}
    First, we consider a fixed sample $S=\{x_1,\dots,x_n\} \subset \Omega$ and let $R > 0$ such that $\Omega\subset B_R^2(0)$. 
    Obviously, 
    \begin{align}
        \Rad_{\rho,n}(\mathcal{G}_{Q,\eps}) \leq   \Rad_{\rho,n}(S_{Q,\eps}) +  \Rad_{\rho,n}(\mathcal{F}_{Q.0}), 
    \end{align}
    where $S_{Q,\eps} = \{\Omega\ni x \mapsto \sup_{B_\eps(x)\cap\Omega} f \,|\, f\in\mathcal{F}_{Q,0}\}$. 
    \Cref{lem:Rademacher bound} provides a bound for the latter term, so it remains to bound the former.
    
    For a function $f\in\mathcal{F}_{Q,0}$ and $\delta > 0$, we can find a representing probability measure $\pi$ such that almost surely, we have $|a|^2 = \|w\|_2^2 \leq Q+\delta$ and $|b| \leq \|w\|_2R$ (see the previous proofs).
    Then, 
    \begin{align}
        \sup_{y\in B_\eps(x)\cap \Omega} f(y) &= \sup_{y\in B_\eps(x)\cap\Omega} \int a(\sigma(w\cdot y + b) -\sigma(b))\de\pi(a,w,b) \\
        &\leq \int \sup_{y\in B_\eps(x)\cap\Omega} a(\sigma(w\cdot y +b)-\sigma(b))\de\pi(a,w,b).
    \end{align}
    Therefore, with $\sigma_i$ for $i=1,\dots,n$ independent Rademacher random variables, we obtain
    \begin{align}
        \sum_{i=1}^n\sigma_i \sup_{B_\eps(x_i)\cap\Omega} f &\leq \int\sum_{i=1}^n \sigma_i \sup_{y\in B_\eps(x_i)\cap\Omega} a(\sigma(w\cdot y+b)-\sigma(b))\de\pi(a,w,b) \\
        &\leq \sup_{\substack{|a|^2=\|w\|_2^2 \leq Q+\delta\\ |b| \leq \|w\|_2R}} \sum_{i=1}^n \sigma_i \sup_{y\in B_\eps(x_i)\cap\Omega} a(\sigma(w\cdot y + b)-\sigma(b)). 
    \end{align}
    In particular, taking the supremum over $f\in\mathcal{F}_{Q,0}$ yields
    \begin{align}
         \sup_{f\in\mathcal{F}_{Q,0}}\sum_{i=1}^n\sigma_i \sup_{B_\eps(x_i)\cap\Omega} f \leq \sup_{\substack{|a|^2=\|w\|_2^2 \leq Q+\delta\\ |b| \leq \|w\|_2R}} \sum_{i=1}^n \sigma_i \sup_{y\in B_\eps(x_i)\cap\Omega} a(\sigma(w\cdot y + b)-\sigma(b)). 
    \end{align}
    Hence, in the following we can restrict ourselves to functions of the form $x\mapsto a(\sigma(w\cdot x + b)-\sigma(b))$, where $|a|^2 = \|w\|_2^2 \leq Q+\delta$ and $|b|\leq \|w\|_2R$. 
    For such functions it is easy to see that due to the monotonicity of the ReLU function, 
    \begin{align}
         \sup_{y\in B_\eps(x_i)\cap\Omega} a(\sigma(w\cdot y + b)-\sigma(b)) \leq a\sigma(w\cdot x_i \pm \eps\|w\|_\star + b)-a\sigma(b),
    \end{align}
    where the sign in front of $\eps\|w\|_\star$ is determined by the sign of $a$. 
    Thus, using the bounds on the magnitudes of $a$ and $b$, we obtain
    \begin{align}
        \Rad_S(S_{Q,\eps}) \leq \frac{1}{n}\IE\left[\sup_{\substack{|a|^2=\|w\|_2^2\leq Q+\delta, \\|b|\leq \|w\|_2R}} \sum_{i=1}^n \sigma_ia\sigma\left(w\cdot x_i \pm \eps\|w\|_\star + b\right)\right] + \frac{(Q+\delta)R}{\sqrt{n}}. 
    \end{align}
    To bound the first term on the right hand side, we proceed exactly as in Step 3 of the proof of \citet[Lemma D.2]{park2023minimum}. 
    First, we notice that the modified bias $b^\prime \coloneq b \pm \eps\|w\|_\star$ satisfies $|b^\prime| \leq \sqrt{Q+\delta}(R+c\eps)$ for some constant $c$, since all norms on $\R^d$ are equivalent.
    Then, we define modified data points $\tilde{x}_i \coloneq (x_i, R+c\eps)$ and modified weights $\tilde{w} \coloneq (w, \frac{b^\prime}{R+c\eps})$ which yields
    \begin{align}
        &\tilde{w}\cdot \tilde{x}_i = w\cdot x_i + b^\prime, \\
        &\|\tilde{x}_i\|_2 = \sqrt{\|x_i\|_2^2 + (R+c\eps)^2} \leq \sqrt{2}(R+c\eps), \\
        &|a|^2+\|\tilde{w}\|_2^2 \leq 2Q + \left(\frac{b^\prime}{R+c\eps}\right)^2 \leq 3(Q+\delta). 
    \end{align}
    Consequently, 
    \begin{align}
        \frac{1}{m}&\IE\left[\sup_{|a|^2=\|w\|_2^2\leq Q+\delta, |b|\leq \|w\|_2R} \left\vert\sum_{i=1}^n \sigma_i a \sigma(w\cdot x_i + b^\prime)\right\vert\right] \\&\leq \frac{1}{n} \IE\left[\sup_{a^2 + |\tilde{w}|^2 \leq 3(Q+\delta)} \left\vert \sum_{i=1}^n \sigma_i a\sigma(\tilde{w}\cdot \tilde{x}_i)\right\vert\right] \\
        & \leq \frac{3(Q+\delta)}{n}\IE\left[\sup_{|u| \leq 1} \sum_{i=1}^n \sigma_i \sigma(u\cdot \tilde{x}_i)\right] \\
        &\leq \frac{C(Q+\delta)}{\sqrt{n}}
    \end{align}
    for some constant $C > 0$ depending on $\Omega$. 
    The last step follows from the Contraction lemma for Rademacher complexities \citep[Lemma 26.9]{shalev-shwartz2014understanding} and a bound for the Rademacher complexity of linear functions \citep[Lemma 26.10]{shalev-shwartz2014understanding}.
    Since $\delta > 0$ was arbitrary, the inequality also holds true without $\delta$. 
    Moreover, as the measure $\rho$ is supported on $\Omega$, we can take the expected value w.r.t. $S\sim\rho^n$ and obtain
    \begin{align}
        \Rad_{\rho,n} (S_{Q,\eps}) \leq \frac{CQ}{\sqrt{n}}. 
    \end{align}
\end{proof}

\begin{remark}
    In context of $\TV_\eps$, we are rather interested in the Rademacher complexity of $\frac{1}{\eps}\mathcal{G}_{Q,\eps}$. 
    Obviously, we immediately get
    \begin{align}
        \Rad_{\rho,n}\left(\frac{1}{\eps}\mathcal{G}_{Q,\eps}\right) \leq \frac{CQ}{\eps\sqrt{n}}. 
    \end{align}
\end{remark}

%%%%%%%%%%%%%%%%%%%%%%%%%%%%%%%%%%%%%%%%%%%%%%%%%%%%%%%%%%%%

\section{Generalization Bounds}\label{sec: generalization bounds}
In this section, we prove generalization bounds for both the empirical risk \eqref{eq: empirical risk} and the empirical nonlocal total variation \eqref{eq: empirical TVeps}. 
The results rely in particular on the Rademacher complexity bounds derived in the previous section. 
We start with a result on the empirical risk. 

\begin{lemma}\label{lem: generalization error}
    Let $\ell\colon\R\times\R\to[0,\infty)$ be measurable, bounded and Lipschitz continuous in the first argument. Then, there exists a constant $C > 0$ such that for all $Q,B \geq 0$ with probability at least $1-\delta$, we have
    \begin{align}
        \sup_{u\in\mathcal{F}_{Q,B}} \left( \mathcal{R}(u)-\widehat{\mathcal{R}}_n(u)\right) \leq C \left(\frac{B+Q}{\sqrt{n}} + \sqrt{\frac{\log(2/\delta)}{n}}\right).
    \end{align}
\end{lemma}

\begin{proof}
    Since the loss function $\ell$ is bounded, measurable and Lipschitz continuous in the first argument, we have
    \begin{align}
        \sup_{u\in\mathcal{F}_{Q,B}} \left( \mathcal{R}(u)-\widehat{\mathcal{R}}_n(u)\right) \leq C \left(\Rad_{\rho_x,n}(\mathcal{F}_{Q,B}) + \sqrt{\frac{\log(2/\delta)}{n}}\right)
    \end{align}
    with probability at least $1-\delta$ for some positive constant $C$ depending on the loss function $\ell$ (see \citet[Theorem 26.5]{shalev-shwartz2014understanding}).
    Combining this with the bound on the Rademacher complexity from \cref{lem:Rademacher bound} yields the desired estimate. 
\end{proof}

Now, we prove a similar result for the empirical nonlocal total variation $\widehat{\TV}_{\eps,n}$. 
As observed when bounding the Rademacher complexity (\cref{lem: Rademacher bound 2}) the function's value at zero cancels out within the $\TV_\eps$ functional. 
Consequently, we can establish the following result for the broader space $\mathcal{F}_{Q,\infty}$ rather than restricting it to $\mathcal{F}_{Q,B}$ as in the previous statement.

\begin{lemma}\label{lem: TV generalization bound}
    Let $Q \geq 0$, $\delta\in (0,1)$ and $\eps\in(0,1]$. Then, with probability at least $1-2\delta$, we have
    \begin{align}
        \sup_{f\in\mathcal{F}_{Q,\infty}} \left| \widehat{\TV}_{\eps,n} (f;\mu) - \TV_{\eps}(f;\mu) \right| \leq CQ\left(\frac{1}{\eps\sqrt{n}} + \frac{\sqrt{\log(2/\delta)}}{\sqrt{2n}}\right)
    \end{align}
    for some constant $C > 0$. 
\end{lemma}

\begin{proof}
    We restrict our detailed analysis to the term corresponding to the label $y=0$, as the argument for $y=1$ is analogous. 
    For a function $f\in\mathcal{F}_{Q,\infty}$, we define
    \begin{align}
        h_f(x) \coloneq \frac{1}{\eps}\left( \sup_{B_{\eps}(x)\cap\Omega} f - f(x)\right).
    \end{align}
    The quantity of interest is
    \begin{align}
        \sup_{f\in\mathcal{F}_{Q,\infty}} \left\vert \frac{1}{n} \sum_{i=1}^n \mathbbm{1}(Y_i=0) h_f(X_i) - \int_\Omega h_f(x) \mu_0(x)\de x \right\vert.
    \end{align}
    Letting $g_f(x,y) \coloneq \mathbbm{1}(y = 0) h_f(x)$ and $\mathcal{G}\coloneq \{g_f\,\vert\, f\in\mathcal{F}_{Q,\infty}\}$, we observe that 
    \begin{align}
        \IE_{(X,y)\sim\mu}[g_f(X,Y)] = \int_\Omega h_f(x) \mu_0(x)\de x. 
    \end{align}
    Defining the abbreviations $\widehat{\IE}_{S_n}[g] \coloneq \frac{1}{n}\sum_{i=1}^n g(X_i,Y_i)$ and $\IE[g] \coloneq \IE_{(X,Y)\sim\mu} [g(X,Y)]$ for $g\in\mathcal{G}$, the target quantity can be written as $\sup_{g\in\mathcal{G}} \vert \widehat{\IE}_{S_n}[g] - \IE[g]\vert$. 
    Since it is easier to handle this term without the absolute value, we additionally define the positive and negative deviations as 
    \begin{align}
        \Phi(S_n) \coloneq \sup_{g\in\mathcal{G}}\left(\widehat{\IE}_{S_n}[g] - \IE[g]\right), \qquad\Psi(S_n) \coloneq \sup_{g\in\mathcal{G}}\left(\IE[g] - \widehat{\IE}_{S_n}[g]\right).
    \end{align}
    By standard Rademacher symmetrization, the expected values with respect to $S_n \sim \mu^n$ of both $\Phi(S_n)$ and $\Psi(S_n)$ are upper bounded by $2\Rad_{n,\mu}(\mathcal{G})$. 
    To relate this to the Rademacher complexity of $\mathcal{G}_{Q,\eps}$ defined in \eqref{eq:def GQeps}, we condition on the labels $Y$ and denote by $N_0$ the number of data points with label zero and $\tilde{X}_1,\dots,\tilde{X}_{N_0}$ the corresponding $x$-components. 
    Then, 
    \begin{align}
        \Rad_{n,\mu}(\mathcal{G}) &= \IE_{(X,Y)\sim\mu^n, \sigma}\left[\sup_{f \in\mathcal{F_{Q,\infty}}} \frac{1}{n}\sum_{i=1}^n \sigma_i \mathbbm{1}(Y_i=0) h_f(X_i)\right] \\
        &= \IE_Y\left[\frac{N_0}{n}\IE_{X,\sigma \vert Y} \left[ \sup_{f\in\mathcal{F}_{Q,\infty}}\frac{1}{N_0} \sum_{i=1}^{N_0} \sigma_i h_f(\tilde{X}_i)\right]\right] \\
        &= \IE_Y \left[\frac{N_0}{n}\frac{1}{\eps}\Rad_{N_0, \frac{\mu_0}{\mu_0(\Omega)}}(\mathcal{G}_{Q,\eps})\right] \\
        &\leq \frac{CQ}{n\eps}\IE_Y\left[\sqrt{N_0}\right].
    \end{align}
    We notice that $N_0 \sim \mathrm{Binomial}(n,\mu_0(\Omega))$ and consequently, using Jensen's inequality, 
    \begin{align}
        \IE_Y[\sqrt{N_0}] \leq \sqrt{\IE[N_0]} = \sqrt{n\mu_0(\Omega)}. 
    \end{align}
Overall, we arrive at 
    \begin{align}
        \IE_{S_n}[\Phi(S_n)] \leq \frac{CQ}{\eps\sqrt{n}}, \qquad \IE[\Psi(S_n)] \leq \frac{CQ}{\eps\sqrt{n}}. 
    \end{align}
    Applying McDiarmid's inequality, which is possible since changing a single pair $(X_i,Y_i)$ in the sample alters both $\Phi(S_n)$ and $\Psi(S_n)$ by at most $\frac{Q}{n}$ as all functions $f\in\mathcal{F}_{Q,\infty}$ are $Q$-Lipschitz, yields
    \begin{align}
        &\mathbb{P}\left( \Phi(S_n) > \IE[\Phi(S_n)] + t \right) \leq \exp\left(-\frac{2nt^2}{Q^2}\right), \\
        &\mathbb{P}\left( \Psi(S_n) > \IE[\Psi(S_n)] + t \right) \leq \exp\left(-\frac{2nt^2}{Q^2}\right).
    \end{align}
    Taking a union bound and using the upper bounds for the expectations of $\Phi(S_n)$ and $\Psi(S_n)$, we obtain
    \begin{align}
        \mathbb{P}\left( \sup_{g\in\mathcal{G}} \vert \widehat{\IE}_{S_n}[g] - \IE[g]\vert > \frac{CQ}{\eps\sqrt{n}} + t\right) \leq 2\exp\left(-\frac{2nt^2}{Q^2}\right).
    \end{align}
    Setting the right hand side to equal $\delta$ leads to $t = \frac{Q\sqrt{\log(2/\delta)}}{\sqrt{2n}}$. 
    Repeating the same argument for the $\mu_1$-term concludes the proof. 
\end{proof}

For our analysis, specifically for the $\liminf$ inequality for the regularizer in \cref{prop: liminf regularizer}, we must handle sequences of functions, where a uniform bound on the Barron semi-norm cannot be guaranteed.
Instead, the Barron semi-norm may grow asymptotically like $\lambda_n/\eps_n$ where $(\eps_n)$ and $(\lambda_n)$ are sequences of positive numbers converging to zero. 
Consequently, our next result establishes almost sure convergence of the empirical nonlocal total variation to its population version, under suitable scalings of these sequences. 

\begin{lemma}\label{lem: empirical TV convergence}
    Let $\eps_n,\lambda_n$ be chosen such that 
    \begin{align}
        \lim_{n\to\infty} \frac{\eps_n\sqrt{\log(n)}}{\lambda_n\sqrt{n}} = 0, \qquad \lim_{n\to\infty} \frac{1}{\lambda_n\sqrt{n}} = 0.
    \end{align}
    Then, almost surely (w.r.t. the drawing of the samples) simultaneously for all sequences $(f_n)\subset \B(\Omega)$ that satisfy $\limsup_{n\to\infty} \frac{\lambda_n}{\eps_n} [f_n]_\B <\infty$, we have 
    \begin{align}
        \lim_{n\to\infty} |\widehat{\TV}_{\eps_n,n} (f_n;\mu)-\TV_{\eps_n}(f_n;\mu)| = 0. 
    \end{align}
\end{lemma}

It is worth noticing that the first scaling condition for $\eps_n$ and $\lambda_n$ is automatically satisfied, when $\lim_{n\to\infty} \frac{\lambda_n}{\eps_n} \in (0,\infty]$, which is implied by our assumption on $\alpha$ in \cref{thm: main}.

\begin{proof}
    For an arbitrary but fixed $m\in\N$, we consider the sequence $Q_n^m = \frac{m\eps_n}{\lambda_n}$. 
    Setting $\delta = \frac{1}{n^2}$ in \Cref{lem: TV generalization bound} yields that with probability at least $1-\frac{2}{n^2}$, we have
    \begin{align}
        \sup_{f\in\mathcal{F}_{Q_n^m,\infty}} \left| \widehat{\TV}_{\eps_n,n} (f;\mu) - \TV_{\eps_n}(f;\mu) \right| \leq \frac{Cm\eps_n}{\lambda_n}\left(\frac{1}{\eps_n\sqrt{n}} + \frac{\sqrt{\log(n)}}{\sqrt{n}}\right)
    \end{align}
    for some constant $C$. 
    Thus, from the Borel-Cantelli lemma, we can conclude that 
    \begin{align}
        \lim_{n\to\infty} \sup_{f\in\mathcal{F}_{Q_n^m,\infty}} |\widehat{\TV}_{\eps_n,n}(f;\mu)-\TV_{\eps_n}(f;\mu)| = 0
    \end{align}
    almost surely with an exceptional set of measure zero that may depend on $m$. 
    Taking the union over all these sets yields a uniform set of measure zero, such that the above limit holds for all $m\in\N$. 
    Since every sequence $(f_n)$ with $\limsup_{n\to\infty} \frac{\lambda_n}{\eps_n}[f_n]_\B < \infty$ satisfies $f_n \in \mathcal{F}_{Q_n^m,\infty}$ for some $m\in\N$, the above considerations prove the statement. 
\end{proof}

%%%%%%%%%%%%%%%%%%%%%%%%%%%%%%%%%%%%%%%%%%%%%%%%%%%%%%%%%%%%%%%%

\section{Results for the Nonlocal Perimeter and Total Variation}\label{sec: nonlocal per and tv}

In this section we essentially extend the Gamma-convergence proof of the nonlocal total variation $\TV_\eps$ for Euclidean norm balls, given by \citet{bungert2024gamma}, to general norm balls. 
Note however that we impose stronger assumptions on the conditional densities $\mu_0$ and $\mu_1$ which strongly simplifies the arguments.
Proving Gamma-convergence has two ingredients:
The first one is the so-called liminf inequality, which demands that whenever $u_n\to u$ in $L^1(\Omega)$ we have
\begin{align*}
    \TV(u;\mu) \leq \liminf_{n\to\infty}\TV_{\eps_n}(u_n;\mu),
\end{align*}
see \cref{lem: liminf TV}.
The second ingredient is the limsup inequality which demands that for any $u\in L^1(\Omega)$ there exists a sequence $(u_n)\subset L^1(\Omega)$ (called recovery sequence) converging to $u$ such that
\begin{align*}
    \limsup_{n\to\infty}\TV_{\eps_n}(u_n;\mu) \leq \TV(u;\mu).
\end{align*}
Note that by virtue of the liminf inequality, equality holds above. 
Thanks to the regularity assumptions in \cref{assum: mu} the proof of this result is pretty standard and follows from a Taylor expansion and a diagonal argument.
In our scenario, where we are interested in minimizers of \eqref{eq: full objective}, we need a slightly weaker result, which  just requires the limsup inequality to be true for Barron functions, see \cref{lem: limsup TV Barron}.

Furthermore, because our subsequent analysis operates with $L^1$ functions rather than just Barron functions, we must slightly modify the definition of $\TV_\eps$. 
Specifically, the supremum and infimum must be replaced by essential supremum and infimum to account for the fact that $L^1$ functions are defined only up to sets of measure zero. 
This leads to 
\begin{align}
    \TV_\eps(f;\mu) \coloneq 
    %\frac{1}{\eps}
    \int_\Omega %\left(\esssup_{B_\eps(x)\cap\Omega} f - f(x)\right) 
    \frac{\esssup_{B_\eps(x)\cap\Omega} f - f(x)}{\eps}
    \mu_0(x)\de x + 
    %\frac{1}{\eps}
    \int_\Omega 
    %\left(f(x)-\essinf_{B_\eps(x)\cap\Omega} f\right)
    \frac{f(x)-\essinf_{B_\eps(x)\cap\Omega} f}{\eps}
    \mu_1(x)\de x.
\end{align}
We note, however, that for continuous functions such as Barron functions, this generalized formulation coincides with our previous definition in \eqref{eq: def TVeps}. 

% DGA: why is this a local perimeter? (we could give a bit of intuition and dummy figures from previous papers. at least for camera-ready, to have a semi-self-contained paper)
With the aim of proving the liminf inequality, let us first use the definition of $\TV_\eps$ in \eqref{eq: def TVeps} to define a nonlocal perimeter through
\begin{align}
    \Per_\eps(A;\mu) \coloneq \TV_\eps(\chi_A;\mu).
\end{align}
for measurable sets $A$, where $\chi_A$ denotes the characteristic function of the set $A$, taking the value $1$ inside $A$ and $0$ outside $A$.
\cite{bungert2023geometry} showed that $\Per_\eps$ indeed qualifies as \emph{perimeter}, being submodular and satisfying a coarea formula.
Gamma-convergence of this perimeter was studied by \citet{bungert2024gamma} in the case where $\|\cdot\|$ is the Euclidean norm. 
Additionally, for abbreviation we denote
\begin{align}
    &\Per_{\eps}^0(A;\mu)\coloneq %\frac{1}{\eps}
    \int_\Omega
    %\left( \esssup_{B_\eps(x)\cap\Omega}\chi_A-\chi_A(x)\right)
    \frac{\esssup_{B_\eps(x)\cap\Omega}\chi_A-\chi_A(x)}{\eps}
    \mu_0(x)\de x \\
    &\Per_\eps^1(A;\mu)\coloneq %\frac{1}{\eps}
    \int_\Omega 
    %\left(\chi_A(x)-\essinf_{B_\eps(x)\cap\Omega}\chi_A \right)
    \frac{\chi_A(x)-\essinf_{B_\eps(x)\cap\Omega}\chi_A}{\eps}
    \mu_1(x)\de x
\end{align}
so that $\Per_\eps(A;\mu) = \Per_\eps^0(A;\mu) + \Per_\eps^1(A;\mu). $

First, we state a useful representation of the nonlocal perimeter, see also \citet{chambolle2012remark}.
For this statement, we need to define the measure theoretic interior of a set $A\subset\R^d$. 
It is given by
\begin{align}
    A^1 \coloneq \left\lbrace x\in\R^d \,\middle\vert \,\lim_{r\to0} \frac{\mathcal{L}^d(B_r^2(x)\cap A)}{\mathcal{L}^d(B_r^2(x))} = 1 \right\rbrace
\end{align}
where $B_r^2(x)$ denotes the open ball with radius $r$ around $x$ with respect to the Euclidean norm $\|\cdot\|_2$. 
We refer to \citet{evans1991measure} for more details. 
\begin{lemma}\label{perrepresentation}
    Let $A\subset\Omega$ be measurable. Then, 
    \begin{align}
        \Per_\eps(A;\mu) = \frac{1}{\eps}\mu_0(\{x\in (A^1)^c \,\vert\, \dist(x,A^1) < \eps\}) + \frac{1}{\eps}\mu_1(\{x\in A^1 \,\vert\, \dist(x,(A^1)^c) < \eps\}).
    \end{align}
\end{lemma}

We also need an auxiliary result which follows from the product rule.
For this, see \citet{ambrosio2000functions} for the definition of the standard total variation $\TV(u)$ for a function of bounded variation $u\in BV(\Omega)$.
Furthermore, $W^{1,1}(\Omega)$ denotes the Sobolev space of absolutely integrable functions whose weak derivatives are also absolutely integrable.
It holds $W^{1,1}(\Omega)\subset BV(\Omega)$ and the inclusion is strict.

\begin{lemma}\label{lem: TV of product}
    Let $\Omega\subset\R^d$ be an open domain with Lipschitz boundary, $u\in W^{1,1}(\Omega)\cap L^\infty(\Omega)$ and $v\in BV(\Omega)\cap L^\infty(\Omega)$ be non-negative. Then, $uv\in BV(\Omega)$ with
    \begin{align}
        \TV(uv) \leq \beta\int_\Omega \|\nabla u\|_\star v \de x + \|u\|_{L^\infty(\Omega)}\TV(v),
    \end{align}
    where $\beta > 0$ is a constant such that $\|\cdot\|_2 \leq \beta \|\cdot\|_\star$. 
\end{lemma}

\begin{proof}
    First, let $\tilde{v}\in BV(\R^d)$ be an extension of $v$ with compact support such that $0\leq \tilde{v}\leq\|v\|_{L^\infty(\Omega)}$. 
    Such a function exists, as there exists an extension operator $E$ such that $Ev\in BV(\R^d)$ and $Eu$ has compact support, since $\Omega$ is an extension domain. 
    Subsequently, we can apply a Lipschitz-continuous truncation operator to guarantee the bounds, which preserves the compact support leads to $\tilde{v}\in BV(\R^d)$ \citep[Theorem 3.96]{ambrosio2000functions}.
    We then consider $v_\eps \coloneq \tilde{v}\ast \varphi_\eps$, where $\varphi_\eps$ is a standard mollifier (i.e. $\varphi_\eps = \eps^{-d} \varphi(\frac{\cdot}{\eps})$ with some $\varphi\in C_c^\infty(\B_1(0)), \varphi \geq 0$ and $\int\varphi = 1$). 
    Then, $v_\eps \to v$ in $L^1(\Omega)$ and $\TV(v_\eps) \to \TV(v)$ \citep[Remark 3.22]{ambrosio2000functions}.
    Moreover, $v_\eps\in C_c^\infty(\R^d)$ and 
    \begin{align}
        \|v_\eps\|_{L^\infty(\R^d)} \leq \|\tilde{v}\|_{L^\infty(\R^d)} \int_{\R^d} \varphi(x) \de x = \|v\|_{L^\infty(\Omega)}
    \end{align}
    as well as 
    \begin{align}
       \lim_{\eps\to 0 }  \|uv_\eps - uv\|_{L^1(\Omega)} \leq \|u\|_{L^\infty(\Omega)} \lim_{\eps\to 0}\|v_\eps - v\|_{L^1(\Omega)} = 0.
    \end{align}
    Hence, we can use the product rule together with the lower semi-continuity of the $\TV$ functional to obtain
    \begin{align}
        \TV(uv) &\leq \liminf_{\eps\to 0}\TV(uv_\eps) = \liminf_{\eps\to 0}\int_\Omega \|\nabla (uv_\eps)\|_2 \de x \\
        &\leq \liminf_{\eps\to 0}\beta\int_\Omega \|\nabla u\|_\star |v_\eps| \de x + \int_\Omega |u| \|\nabla v_\eps\|_2 \de x \\
        &\leq \liminf_{\eps\to 0}\beta\int_\Omega \|\nabla u\|_\star |v_\eps| \de x +  \|u\|_{L^\infty(\Omega)} \TV( v_\eps). 
    \end{align}
    Finally, (after extracting a subsequence) $v_\eps$ converges to $v$ pointwise almost everywhere on $\Omega$ and $\|\nabla u \|_\star v_\eps$ is bounded by the integrable function $\|\nabla u\|_\star v$. Thus, the dominated convergence theorem yields
    \begin{align}
        \lim_{\eps\to 0} \int_\Omega \|\nabla u\|_\star v_\eps \de x = \int_\Omega \|\nabla u\|_\star v \de x
    \end{align}
    concluding the proof. 
\end{proof}

To prove the following liminf inequality for the perimeters---which will easily imply the desired result for the total variations---we proceed similar to \citet{bungert2024gamma,siethoff2025existence}. 
We start with a compactness result. 
\begin{lemma}\label{lem: per compactness}
    Let \cref{assum: mu} be satisfied. Then, for any sequence of positive numbers $(\eps_k)$ converging to zero and measurable sets $(A_k)$ with $\limsup_{k\to\infty}\Per_{\eps_k}(A_k;\mu) < \infty$ there exists a subsequence $(A_{n_j})$ and a set $A$ with $\TV(\chi_A)<\infty$ such that 
    \begin{align}
        \chi_{A_k} \to \chi_A \text{ in } L^1(\Omega).
    \end{align}
\end{lemma}

\begin{proof}
    We define the sequences $(u_k)$ and $(v_k)$ by
    \begin{align}
        u_k(x) \coloneq \max\left\lbrace 1-\frac{\dist(x,A_k^1)}{\eps_k}, 0\right\rbrace \qquad v_k(x) \coloneq \min\left\lbrace \frac{\dist(x,(A_k^1)^c)}{\eps_k}, 1\right\rbrace
    \end{align}
    for $x\in\R^d$.
It holds that
    \begin{align}
        \|\nabla u_k \|_\star = \frac{1}{\eps_k} \chi_{\{x\in\R^d \,\vert\,0<\dist(x,A_k^1) < \eps_k\}}, \qquad  \|\nabla v_k \|_\star = \frac{1}{\eps_k} \chi_{\{x\in\R^d \,\vert\,0<\dist(x,(A_k^1)^c) < \eps_k\}}.
    \end{align}
    almost everywhere (see \citet[p. 9]{chambolle2012remark}). 
    Therefore, we can deduce that
    \begin{align}
        \Per_{\eps_k}^0(A_k;\mu) \geq \int_\Omega \|\nabla u_k\|_\star \mu_0(x)\de x, 
    \end{align}
    since $\{x\in\Omega\,\vert\, 0 < \dist(x,A_k^1)<\eps_k\} \subset \{x\in(A_k^1)^c\,\vert\, \dist(x,A_k^1) < \eps_k\}$.
    Analogously, we obtain
    \begin{align}
        \Per_{\eps_k}^1(A_k;\mu) \geq \int_\Omega \|\nabla v_k\|_\star \de \mu_1(x).
    \end{align}
    Since the densities $\mu_0$ and $\mu_1$ are continuous up to the boundary and $\mu_0+\mu_1 > 0$ on $\overline{\Omega}$, these functions can be extended to functions $\tilde{\mu}_i \in C(\tilde{\Omega})$, where $\tilde{\Omega}$ is an open neighborhood of $\overline{\Omega}$ and additionally $\tilde{\mu}_0+\tilde{\mu}_1 \geq \delta$ for some $\delta > 0$. 
    We now consider, for $i\in \{0,1\}$, 
    \begin{align}
        \tilde{\Omega}_i \coloneq \left\{x\in\tilde{\Omega} \,\vert\, \mu_i(x) > \frac{\delta}{3}\right \}
    \end{align}
    which defines open sets such that $\overline{\Omega}\subset \tilde{\Omega}_0\cup\tilde{\Omega}_1$. 
    In particular, there exists a partition of unity $\Psi_0,\Psi_1$ subordinate to this covering, meaning that $\Psi_0,\Psi_1$ are smooth functions satisfying $0\leq \Psi_i \leq 1$, $\supp\Psi_i\subset\tilde{\Omega}_i$ and $\Psi_0+\Psi_1 = 1$ on $\overline{\Omega}$. 
    \Cref{lem: TV of product} implies
    \begin{align}
        \TV(u_k \Psi_0) &\leq \beta\int_\Omega \|\nabla u_k\|_\star \Psi_0\de x + \|u_k\|_{L^\infty(\Omega)}\TV(\Psi_0) \\
        &\leq \frac{3\beta}{\delta}\int_\Omega \|\nabla u_k\|_\star \mu_0 \de x + \TV(\Psi_0).
    \end{align}
    Consequently, 
    \begin{align}
        \limsup_{k\to\infty}\TV(u_k \Psi_0) \leq \limsup \frac{3\beta}{\delta}\Per_{\eps_k}^0(A_k;\mu) + \TV(\Psi_0) < \infty
    \end{align}
    and we can apply compactness in $BV(\Omega) $ \citep[Theorem 3.23]{ambrosio2000functions} to extract a converging subsequence (without relabeling) such that $u_k\Psi_0$ converges in $L^1(\Omega)$ to some $u\in BV(\Omega)$. 
    Similarly, we find that $v_k\Psi_1$ converges in $L^1(\Omega)$ to a function $v\in BV(\Omega)$.
    We finalize the proof by showing that $\chi_{A_k}$ converges to $u+v$. 
    To see this, we first observe that 
    \begin{align}
        \frac{\delta}{3} \int_{\tilde{\Omega}_0 \cap\Omega} |u_k - \chi_{A_k}| \de x &= \int_{\tilde{\Omega}_0 \cap\Omega} \left(1-\frac{\dist(x,A_k^1)}{\eps_k} \right) \chi_{\{x\in\Omega \,|\, 0 < \dist(x,A_k^1 < \eps_k\}}(x) \de x \\
        &\leq \int_{\tilde{\Omega}_0 \cap\Omega} \chi_{\{x\in\Omega\,|\, 0 < \dist(x,A_k^1 < \eps_k\}}(x) \mu_0(x) \de x \\
        &\leq \eps_k \Per_{\eps_k}^0(A_k;\mu) \to 0, \quad \text{ as } k\to\infty.
    \end{align}
    Using the same ideas, we can argue that $\int_{\tilde{\Omega}_1 \cap\Omega} |v_k - \chi_{A_k}| \de x \to 0$ as $k\to\infty$.
    Hence, using that $\Psi_0+\Psi_1 = 1$ on $\Omega$, we obtain that 
    \begin{align}
        \int_\Omega |\Psi_0u_k + \Psi_1v_k - \chi_{A_k}|\de x \leq \int_{\tilde{\Omega}_0\cap\Omega} \Psi_0 |u_k-\chi_{A_k}|\de x + \int_{\tilde{\Omega}_1\cap\Omega} \Psi_1 |v_k-\chi_{A_k}| \de x
    \end{align}
    and the right hand side converges to zero as $k\to\infty$. 
    Consequently, $\chi_{A_k}$ must have the same limit in $L^1(\Omega)$ as $\Psi_0u_k+\Psi_1v_k$ which is $u+v\in BV(\Omega)$. 
    Since $\chi_{A_k}$ only takes the values $0$ and $1$, the limit must do so too, meaning that there exists a measurable set $A$ such that $\chi_A = u+v$ and in particular, $\TV(\chi_A)<\infty$. 
\end{proof}
We establish the $\liminf$ inequality using similar arguments.

\begin{lemma}\label{lem:liminf perimeter}
Let \cref{assum: mu} be satisfied.
Let $(A_k)$ be a collection of measurable subsets of $\Omega$ and $A\subset \Omega$ measurable such that $\chi_{A_k} \to \chi_A$ in $L^1(\Omega)$. Then, for any sequence of positive numbers $(\eps_k)$ converging to zero it holds that
    \begin{align}
        \liminf_{k\to\infty}\Per_{\eps_k}(A_k; \mu) \geq \Per(A;\mu),
    \end{align}
    where we define
    \begin{align*}
        \Per(A;\mu) \coloneq 
        \begin{dcases}
        \int_{\partial^\star A \cap\Omega}(\mu_0+\mu_1)\left\rVert \frac{\de D\chi_A}{\de|D\chi_A|}\right\rVert_\star\de\mathcal{H}^{d-1}
        \quad&\text{if }\chi_A \in BV(\Omega),
        \\
        +\infty\quad&\text{else.}
        \end{dcases}
    \end{align*}
\end{lemma}

To clarify our notation, we recall that for a function $u\in BV(\Omega)$ the distributional derivative  can be represented by a finite $\R^d$-valued Radon measure $Du$ \citep{ambrosio2000functions}. 
Furthermore, we denote the associated total variation measure as $|Du|$ and $\mathcal{H}^{d-1}$ is the $(d-1)$-dimensional Hausdorff measure. 

\begin{proof}
    We first note that w.l.o.g. we can assume 
    \begin{align}
        \liminf_{k\to\infty}\Per_{\eps_k}(\chi_{A_k}; \mu) = \lim_{k\to\infty}\Per_{\eps_k}(\chi_{A_k}; \mu) < \infty
    \end{align}
    by extracting a subsequence and observing that the inequality is trivially satisfied if this limit is infinite. 
    From \cref{lem: per compactness} we can conclude that $\chi_A\in BV(\Omega)$.
    We define the sequences $(u_k)$ and $(v_k)$ as in the previous proof via
    \begin{align}
        u_k(x) \coloneq \max\left\lbrace 1-\frac{\dist(x,A_k^1)}{\eps_k}, 0\right\rbrace \qquad v_k(x) \coloneq \min\left\lbrace \frac{\dist(x,(A_k^1)^c)}{\eps_k}, 1\right\rbrace
    \end{align}
    for $x\in\R^d$.
    As before (see the proof of \cref{lem: per compactness}) it holds that 
    \begin{align}
        \Per_{\eps_k}^0(A_k;\mu) \geq \int_\Omega \|\nabla u_k\|_\star \mu_0(x)\de x.
    \end{align}
    Let $(\delta_n)$ be a sequence of positive real numbers that converge to zero as $n$ approaches infinity. 
    We define the sets
    \begin{align}
        \Omega_{\delta_n}^i \coloneq \{\mu_i > \delta_n\}
    \end{align}
    for $i\in\{0,1\}$, which are open thanks to \cref{assum: mu}.  
    Since we assumed that the sequence $(\Per_{\eps_k}^0(A_k;\mu))$ is bounded by some constant $C$, we can conclude 
    \begin{align}
        |Du_k|(\Omega_{\delta_n}^0) &= \int_{\Omega_{\delta_n}^0} \|\nabla u_k(x)\|_2 \de x \leq \beta \int_{\Omega_{\delta_n}^0} \|\nabla u_k(x)\|_\star \de x \\
        &\leq \frac{\beta}{\delta_n}\int_\Omega \|\nabla u_k(x)\|_\star  \mu_0(x)\de x \leq \frac{\beta}{\delta_n} \Per_{\eps_k}^0(A_k;\mu) \leq \frac{C\beta}{\delta_n}
    \end{align}
    using the norm equivalence $\|\cdot\|_2\leq \beta \|\cdot\|_\star$ and the fact that $\mu_0$ is bounded from below by $\delta_n$ on $\Omega_{\delta_n}^0$. 
    By \citet[eq (3.11)]{ambrosio2000functions}, together with the convergence of $u_k$ to $\chi_A$ in $L^1(\Omega_{\delta_n}^0)$ which follows as in the proof of \cref{lem: per compactness}, this implies the weak* convergence of $Du_k$ to $D\chi_A$ on $\Omega_{\delta_n}^0$. 
    Hence, applying the Reshetnyak lower semicontinuity theorem (see e.g. \citet[Theorem 2.38]{ambrosio2000functions}) yields 
    \begin{align}
        \liminf_{k\to\infty} \Per_{\eps_k}^0 (A_k;\mu) &\geq \liminf_{k\to\infty}\int_{\Omega_{\delta_n}^0} \|\nabla u_k\|_\star \mu_0 \de x \\
        &\geq \int_{\Omega_{\delta_n}^0} \left\Vert \frac{\de D\chi_A}{\de |D\chi_A|}\right\Vert_\star\mu_0 \de |D\chi_A|.
    \end{align}
    Since this holds for every $n\in\N$, we can take the limit $n\to\infty$ and use the dominated convergence theorem to obtain
    \begin{align}
        \liminf_{k\to\infty}\Per_{\eps_k}^0 (A_k;\mu) &\geq \lim_{n\to\infty} \int_{\Omega_{\delta_n}^0} \left\Vert \frac{\de D\chi_A}{\de |D\chi_A|}\right\Vert_\star\mu_0 \de |D\chi_A| \\
        &= \int_\Omega \left\Vert \frac{\de D\chi_A}{\de |D\chi_A|}\right\Vert_\star\mu_0 \de |D\chi_A| \\
        &= \int_{\partial^\star A\cap\Omega} \left\|\frac{\de D\chi_A}{\de|D\chi_A|}\right\|_\star \mu_0\de \mathcal{H}^{d-1}.
    \end{align}
    Repeating the arguments for $v_k$ instead of $u_k$ and $\mu_1$ instead of $\mu_0$ concludes the proof.
\end{proof}

In fact, a corresponding result for the nonlocal total variation follows directly from the following coarea formula \citep{bungert2023geometry}:
\begin{align}
    \TV_\eps(f;\mu) = \int_\R \Per_\eps(\{f\geq t\};\mu)\de t. 
\end{align}

\begin{lemma}\label{lem: liminf TV}
    Let \cref{assum: mu} be satisfied.
    Let $(u_k)\subset BV(\Omega)$ and $u\in BV(\Omega)$ such that $u_k\to u$ in $L^1(\Omega)$. Then, for any sequence of positive numbers $(\eps_k)$ converging to zero, it holds that 
    \begin{align}
        \TV(u;\mu) \leq \liminf_{k\to\infty}\TV_{\eps_k}(u_k;\mu).
    \end{align}
\end{lemma}

\begin{proof}
Without loss of generality, by passing to a subsequence if necessary, we can assume that $\liminf_{k\to\infty} \TV_{\eps_k}(u_k;\mu) = \lim_{k\to\infty} \TV_{\eps_k}(u_k;\mu)$. 
    Moreover, there exists a subsequence $(u_k)$ (without relabeling) such that $\{u_k > t\} \to \{u>t\}$ in $L^1(\Omega)$ for almost every $t\in\R$.
    Consequently, using the coarea formula, Fatou's lemma, and \cref{lem:liminf perimeter} we obtain
    \begin{align}
        \liminf_{k\to\infty}\TV_{\eps_k}(u_k;\mu) &= \liminf_{k\to\infty} \int_{-\infty}^\infty \Per_{\eps_k}(\{u_k > t\};\mu) \de t \\
        &\geq  \int_{-\infty}^\infty \liminf_{k\to\infty}\Per_{\eps_k}(\{u_k > t\};\mu) \de t \\
        &\geq \int_{-\infty}^\infty \Per(\{u>t\};\mu)\de t = \TV(u;\mu). 
    \end{align}
\end{proof}

Next, we prove that for Barron functions $f$, $\TV_\eps(f;\mu)$ converges to $\TV(f;\mu)$ as $\eps$ tends to zero. 

\begin{lemma}\label{lem: limsup TV Barron}
    Let $f\in\B(\Omega)$ and $(\eps_n)_{n\in\N}$ a sequence of positive numbers converging to zero. Then,
    \begin{align}
        \lim_{n\to\infty} \TV_{\eps_n}(f;\mu) = \TV(f;\mu).
    \end{align}
\end{lemma}

\begin{proof}
    Let $\eps > 0$ arbitrary. 
    By \cref{lem:lipschitz_barron} the function $f$ is Lipschitz continuous, hence it is differentiable almost everywhere by Rademacher's theorem. 
    Consequently, for almost every $x\in\Omega$ with $\dist(x,\partial\Omega) > \eps$, we have
    \begin{align}
        f(x+\eps u) - f(x) = \eps\nabla f(x)\cdot u + o_x(\eps)
    \end{align}
    for $u\in B_1(0)$. In particular, taking the supremum with respect to $u$, we obtain
    \begin{align}
        \esssup_{B_\eps(x)\cap\Omega} f - f(x) = \eps\|\nabla f(x)\|_\star + o_x(\eps).
    \end{align}
    Taking the limit $\eps\to 0$, we obtain that 
    \begin{align}
        \lim_{\eps \to 0} \frac{1}{\eps}\left( \esssup_{B_\eps(x)\cap\Omega}f-f(x)\right) = \|\nabla f(x)\|_\star
    \end{align}
    for almost every $x\in\Omega$. 
    Moreover, $\frac{1}{\eps}\left( \esssup_{B_\eps(x)\cap\Omega}f-f(x)\right)$ is dominated by a constant which is integrable over $\Omega$. 
    Thus, with Lebesgue's dominated convergence theorem, we can conclude 
    \begin{align}
        \lim_{n\to\infty}\int_\Omega \frac{1}{\eps_n}\left( \esssup_{B_{\eps_n}(x)\cap\Omega}f-f(x)\right)\mu_0(x)\de x = \int_\Omega\|\nabla f(x)\|_\star\mu_0(x)\de x.
    \end{align}
    Repeating the same arguments for the $\mu_1$-term yields the desired result. 
\end{proof}

%%%%%%%%%%%%%%%%%%%%%%%%%%%%%%%%%%%%%%%%%%%%%%%%%%%%

\section{Results for the Full Regularizer}\label{sec: regularizer}
In the previous section we studied the nonlocal total variation and proved the liminf inequality as well as a version of the limsup inequality.
Next, we need analogous results for the regularizer in \eqref{eq: full objective} which requires two key changes:
First, we need to consider the empirical version of the nonlocal total variation and second the weight decay regularizer needs to be added to the functional to ensure existence of minimizers and compactness of sequences of minimizers. 
As already mentioned in \cref{sec: proof details}, it is useful to additionally divide the entire functional by $\eps_n$, leading to
\begin{align}
    R_n(\theta) \coloneq \widehat{\TV}_{\eps_n,n}(f_\theta;\mu)
    +
    \frac{\lambda_n}{\eps_n} R_{WD}(\theta). 
\end{align}
The liminf inequality for this regularizer functional is proved in \cref{prop: liminf regularizer}, and we refer to \cref{prop: limsup regularizer} for the limsup inequality.

\begin{proposition}\label{prop: liminf regularizer}
    Let \cref{assum: mu} be satisfied.
    Let $(f_n)\subset \B(\Omega)$ such that $f_n\to f \in BV(\Omega)$
    %\todo{strictly speaking you cannot assume this but need to show first that the limit is in BV, using compactness} 
    in $L^1(\Omega)$. Moreover, assume that 
    \begin{align}
        \lim_{n\to\infty} \frac{\eps_n}{\lambda_n}\sqrt{\frac{\log(n)}{n}} = 0,\qquad \lim_{n\to\infty}\frac{1}{\lambda_n\sqrt{n}} = 0, \qquad \alpha \coloneq \lim_{n\to\infty} \frac{\lambda_n}{\eps_n} \in [0,\infty).
    \end{align}
    Then, 
    \begin{align}
        \TV(f;\mu) + \alpha[f]_\B \leq \liminf_{n\to\infty} \widehat{\TV}_{\eps_n,n}(f_n;\mu) + \frac{\lambda_n}{\eps_n}[f_n]_\B. 
    \end{align}
    almost surely (with respect to the samples). 
    In the case where $\alpha = 0$, we use the convention that $\alpha\cdot[f]_\B = 0$ even if $[f]_\B = \infty$. 
\end{proposition}

\begin{proof}
    Without loss of generality, by passing to a subsequence, we can assume that 
    \begin{align}
        \liminf_{n\to\infty}  \widehat{\TV}_{\eps_n,n}(f_n;\mu) + \frac{\lambda_n}{\eps_n}[f_n]_\B &= \lim_{n\to\infty} \widehat{\TV}_{\eps_n,n}(f_n;\mu) + \frac{\lambda_n}{\eps_n}[f_n]_\B. 
    \end{align}
    Moreover, if $\limsup_{n\to\infty}\frac{\lambda_n}{\eps_n}[f_n]_\B = \infty$, the claim becomes trivial by extracting another subsequence, where the limsup becomes a limit. 
    Hence, in the following, we consider the case, where $\limsup_{n\to\infty}\frac{\lambda_n}{\eps_n}[f_n]_\B$ is finite. 
    By inserting zero, we can rewrite 
    \begin{align}
        \widehat{\TV}_{\eps_n,n}(f_n;\mu) + \frac{\lambda_n}{\eps_n}[f_n]_\B = \TV_{\eps_n}(f_n;\mu) + \widehat{\TV}_{\eps_n,n}(f_n;\mu) - \TV_{\eps_n}(f_n;\mu)
        +
        \frac{\lambda_n}{\eps_n}[f_n]_\B
    \end{align}
     and \Cref{lem: empirical TV convergence} implies that the difference of the two middle terms converges to zero almost surely as $n\to\infty$. 
     We can additionally use the lower semi-continuity of the Barron semi-norm and the liminf inequality from \Cref{lem: liminf TV} to conclude
     \begin{align}
         \liminf_{n\to\infty}  \widehat{\TV}_{\eps_n,n}(f_n;\mu)+\frac{\lambda_n}{\eps_n}[f_n]_\B  &= \liminf_{n\to\infty} \TV_{\eps_n}(f_n;\mu)+\frac{\lambda_n}{\eps_n}[f_n]_\B \\
         &\geq   \liminf_{n\to\infty}\TV_{\eps_n} (f_n;\mu) +
         \liminf_{n\to\infty}
         \frac{\lambda_n}{\eps_n}[f_n]_\B\\
         &\geq \TV(f;\mu) + \alpha [f]_\B.
     \end{align}
\end{proof}

\begin{remark}
    In particular if the functions $f_n$ are neural networks parametrized by $\theta_n\in\Theta_{m_n}$, then 
    \begin{align}
        \TV(f;\mu) + \alpha[f]_\B \leq \liminf_{n\to\infty}  \widehat{\TV}_{\eps_n,n}(f_n;\mu)+\frac{\lambda_n}{\eps_n}R_{WD}(\theta_n), 
    \end{align}
    since $R_{WD}(\theta_n) \geq [f_n]_\B$ (see \Cref{rmk: nn}). 
\end{remark}

From \Cref{lem: limsup TV Barron}, we can immediately deduce the following equality. 
\begin{lemma}\label{lem: limsup continuous functional Barron}
    For $f\in\B(\Omega)$ it holds that 
    \begin{align}
        \limsup_{n\to\infty}  \TV_{\eps_n}(f;\mu) + \frac{\lambda_n}{\eps_n}[f]_\B = \TV(f;\mu)+\alpha [f]_\B, 
    \end{align}
    where $\alpha \coloneq \limsup_{n\to\infty}\frac{\lambda_n}{\eps_n}$. 
\end{lemma}

It remains to close the gap to obtaining a  sequence of neural networks and considering the empirical version of the non-local total variation.  

\begin{proposition}\label{prop: limsup regularizer}
    Let $f\in\B(\Omega)$ and let 
    \begin{align}
        \lim_{n\to\infty}\frac{1}{\eps_n \sqrt{m_n}} = 0, \qquad \lim_{n\to\infty} \frac{1}{\eps_n\sqrt{n}} = 0, \qquad \limsup_{n\to\infty}\frac{\lambda_n}{\eps_n} =: \alpha.
    \end{align}
    Then, there exists a sequence of neural networks $f_n$ with $m_n$ neurons parametrized by $\theta_n\in\Theta_{m_n}$ such that $f_n\to f$ in $L^1(\Omega)$ and 
    \begin{align}
        \limsup_{n\to\infty}  \widehat{\TV}_{\eps_n,n}(f_n; \mu)+\frac{\lambda_n}{\eps_n}R_{WD}(\theta_n) \leq  \TV(f;\mu)+\alpha [f]_\B. 
    \end{align}
    almost surely (with respect to the samples). 
\end{proposition}

\begin{proof}
    From \Cref{lem: limsup continuous functional Barron}, we infer that
    \begin{align}
         \limsup_{n\to\infty}  \TV_{\eps_n} (f;\mu) + \frac{\lambda_n}{\eps_n} [f]_\B = \TV(f;\mu)+\alpha [f]_\B.
    \end{align}
    We will approximate the function $f$ by neural networks. 
    Precisely, for $n\in\N$, there exists a neural network with $m_n$ neurons parametrized by $\theta_n$ such that 
    \begin{align}
        \| f_n-f\|_{L^\infty(\Omega)} \leq \frac{C}{\sqrt{m_n}}, \qquad R_{WD}(\theta_n) \leq [f]_{\B(\Omega)}, 
    \end{align}
    where the constant $C$ depends on the Barron semi-norm of $f$ as well as the domain $\Omega$, see \Cref{lem:nn_approx_in_l_infty}.
    Hence, 
    \begin{align}
        |\TV_{\eps_n}(f;\mu)-\TV_{\eps_n}(f_n;\mu)| \leq \frac{C}{\eps_n}\|f_n-f\|_{L^\infty(\Omega)} \leq \frac{C}{\eps_n \sqrt{m_n}}. 
    \end{align} 
    Moreover, we can use \Cref{lem: TV generalization bound} to obtain 
    \begin{align}
        |\widehat{\TV}_{\eps_n,n}(f_n;\mu) - \TV_{\eps_n}(f_n;\mu)| \leq \frac{C\sqrt{\log(2/\delta)}}{\sqrt{n}} + \frac{C}{\eps_n\sqrt{n}}
    \end{align}
    with probability at least $1-2\delta$. 
    Hence, combining all the estimates with \Cref{lem: limsup continuous functional Barron} and using the Borel--Cantelli lemma, we obtain
    \begin{align}
        \limsup_{n\to\infty}  \widehat{\TV}_{\eps_n,n}(f_n;\mu)+\frac{\lambda_n}{\eps_n}R_{WD}(\theta_n) &= \limsup_{n\to\infty} \Big(\TV_{\eps_n}(f;\mu) + \TV_{\eps_n}(f_n;\mu) - \TV_{\eps_n}(f;\mu)
        \\
        &\quad 
        +\widehat{\TV}_{\eps_n,n}(f_n;\mu) - \TV_{\eps_n}(f_n;\mu)+\frac{\lambda_n}{\eps_n}R_{WD}(\theta_n)\Big) \\
        &\leq \TV(f;\mu)+\alpha [f]_\B
    \end{align}
    almost surely. 
\end{proof}

% DGA: is this accurate? naively, it seems that some constants depends on the diameter of \Omega, which scales as \sqrt{d} if \Omega=[0,1]^d if I take the example that appears earlier in the appendix
% LB: constants do depend on the dimension, rates don't, i.e., the relative scalings in terms of n
% DGA: I am not sure I understand what you mean, if I look at the next display and C depends on d, then the rate does depends on d?
% LB: I mean the conditions that the right hand side of the display imposes on the other parameters, most importantly \eps_n. 
% For us it suffices that \eps_n \gg \frac{1}{\sqrt{m_n}}.
% If the RHS of the display was of the usual form \frac{1}{m_n^\frac{1}{d}} we would get the condition \eps_n \gg \frac{1}{m_n^\frac{1}{d}} which requires one to choose \eps_n huge in high dimensions 
We note that our analysis benefits significantly from the homogeneous Barron space, which guarantees approximation rates independent of the ambient dimension $d$.

\begin{remark}\label{rmk: limsup rates}
    The sequence $(f_n)$ arising from the previous proof has the rate of convergence
    \begin{align}
        \|f_n-f\|_{L^\infty(\Omega)} \leq \frac{C}{\sqrt{m_n}}
    \end{align}
    for some constant $C$ depending on $f$ and $\Omega$ and satisfies $[f_n] \leq R_{WD}(\theta_n) \leq [f]_\B$. 
\end{remark}

%%%%%%%%%%%%%%%%%%%%%%%%%%%%%%%%%%%%%%%%%%%%%%%%%%%

\section{Proofs of \texorpdfstring{\Cref{prop: liminf AT,prop: limsup AT,prop: compactness}}{Propositions 4.1 to 4.3}}\label{sec: proof of propositions}
In this section, we carry out the proof of the main auxiliary results for proving \cref{thm: main}. 
In the language of Gamma-convergence, \cref{prop: limsup AT} is the limsup inequality, \cref{prop: liminf AT} is the liminf equality, and together they imply the following Gamma-convergence theorem.
\begin{theorem}\label{thm: gamma}
Under the conditions of \cref{thm: main} define the functionals
\begin{align*}
    G_n : L^1(\Omega)\to[0,\infty],
    \quad
    G_n(u)
    :=\inf\left\lbrace F_n(\theta)\,\vert\,\theta\in\Theta_{m_n}\text{ with } u = f_\theta\right\rbrace,
    \qquad
    n\in\N,
\end{align*}
where $f_\theta$ for $\theta\in\Theta_{m_n}$ denotes the neural network with parameters $\theta$.
As $n\to\infty$ the functionals~$G_n$ Gamma-converge to $F$, defined in \eqref{eq: def F}, in the $L^1(\Omega)$ topology.
\end{theorem}
\begin{proof}
Notice that the limsup inequality for this functional follows directly from \Cref{prop: limsup AT} since $G_n(u) \leq F_n(\theta)$ if $\theta$ parametrizes $u$.
For the liminf inequality one uses \Cref{prop: liminf AT} with a parametrization $\theta_n$ of $f_n$ that satisfies
\begin{align}
    F_n(\theta_n) \leq G_n(f_n) + \frac{1}{n}
    \qquad
    \forall n\in\N.
\end{align}
\end{proof}

We proceed with proving the $\liminf$ inequality \Cref{prop: liminf AT}.

\begin{proof}[Proof of \Cref{prop: liminf AT}]
    We first consider the case that $f\in\B$ and $\mathcal{R}(f) = \inf_\B \mathcal{R}$.
    In this case, by \Cref{prop: liminf regularizer}, 
    \begin{align}
        \liminf_{n\to\infty} F_n(f_n) &\geq \liminf_{n\to\infty} \frac{\lambda_n}{\eps_n}R_{WD}(\theta_n) + \widehat{\TV}_{\eps_n,n} (f_n; \mu) \\
        &\geq \liminf_{n\to\infty}\frac{\lambda_n}{\eps_n}[f_n]_\B + \widehat{\TV}_{\eps_n,n}(f_n;\mu) \\
        &\geq \alpha [f]_\B + \TV(f;\mu)
    \end{align}
    almost surely.
    In every other case, we need to prove that $\liminf_{n\to\infty} F_n(f_n) = \infty$. 
    If $f\notin \B(\Omega)$, we know by \Cref{lem: semi norm lower semicontinuous} that $\liminf_{n\to\infty} [f_n]_\B = \infty$ and consequently, 
    \begin{align}
        \liminf_{n\to\infty} F_n(f_n) \geq \liminf_{n\to\infty}\frac{\lambda_n}{\eps_n}[f_n]_\B = \infty. 
    \end{align}
    The last remaining case is $f\in\B$ with $\mathcal{R}(f) > \inf_\B\mathcal{R}$.
    In particular, there exists $\delta > 0$ such that $\mathcal{R}(f_n) - \inf_\B\mathcal{R} > \delta$ for $n$ large enough. 
    In this case, for convenience, we extract a subsequence such that the liminf under consideration becomes a limit. 
    If then $\limsup_{n\to\infty}[f_n]_\B = \infty$, passing to the corresponding subsequence, the inequality follows trivially. 
    Hence, we can assume that $\limsup_{n\to\infty} [f_n]_\B < \infty$, in particular, the Barron semi-norms are bounded by some constant $Q > 0$. 
    Moreover, we can conclude that $f_n(0)$ needs to converge to $f(0)$ due to the convergence in $L^1$ and the fact that all the functions $f_n$ and $f$ are Lipschitz continuous with uniformly bounded Lipschitz constants. 
    To see this, we assume the contrary, i.e. $|f_n(0)-f(0)| \geq \gamma > 0$ and conclude that there exists an open set $U$ (independent of $n$) such that $|f_n - f| \geq \frac{\gamma}{2}$ on $U$, which contradicts the fact that $f_n$ converges to $f$ in $L^1(\Omega)$. 
    Overall, there exists a constant $B>0$ such that $f_n, f\in\mathcal{F}_{Q,B}$ as defined in \eqref{eq: def FQB}. 
    Therefore, invoking \Cref{lem: generalization error}, we obtain
    \begin{align}
        F_n(f_n) &\geq \frac{\widehat{\mathcal{R}}_n(f_n) - \mathcal{R}(f_n)}{\eps_n} + \frac{\mathcal{R}(f_n)-\inf_\B\mathcal{R}}{\eps_n} \\
        &\geq -\frac{C}{\eps_n}\sqrt{\frac{\log(n)}{n}} + \frac{\delta}{\eps_n},
    \end{align}
    with probability at least $1-\frac{1}{n^2}$.
    Since the first term  is bounded, while the latter is unbounded, using the Borel--Cantelli lemma, we can conclude that 
    \begin{align}
        \liminf_{n\to\infty} F_n(f_n) = \infty
    \end{align}
    almost surely. 
\end{proof}

The proof of the $\limsup$ inequality \Cref{prop: limsup AT} uses similar ideas, combining \Cref{prop: limsup regularizer} with the generalization bound \Cref{lem: generalization error}

\begin{proof}[Proof of \Cref{prop: limsup AT}]
    The only case of interest is $f\in\B(\Omega)$ such that $\mathcal{R}(f) = \inf_\B\mathcal{R}$. 
    In every other case, we only need to show that $f\in L^1(\Omega)$ can be approximated by neural networks, because the inequality becomes trivial. 
    The existence of an approximating sequence of neural networks follows directly from \citet{HORNIK1991251}. 
    Hence, in the following we assume that $f\in\B$ with $\mathcal{R}(f) = \inf_\B\mathcal{R}$. 
    We consider the recovery sequence of neural networks $(f_n)$ with corresponding parameters $\theta_n$ for the regularizer term given by \Cref{prop: limsup regularizer}.
    We note that by \Cref{rmk: limsup rates}, we can further assume that 
    \begin{align}
         \|f_n-f\|_{L^\infty(\Omega)} \leq \frac{C}{\sqrt{m_n}} \qquad \text{and} \qquad [f_n]_\B \leq [f]_\B
    \end{align}
    for a positive constant $C$ depending on $f$. 
    Moreover,
    \begin{align}
        F_n(f_n) &= \frac{\widehat{\mathcal{R}}_n(f_n) - \inf_\B \mathcal{R}}{\eps_n} + \frac{\lambda_n}{\eps_n}R_{WD}(\theta_n) + \widehat{\TV}_{\eps_n,n} (f_n;\mu) \\
        &= \frac{\widehat{\mathcal{R}}_n(f_n) - \mathcal{R}(f_n)}{\eps_n} + \frac{\mathcal{R}(f_n) - \inf_\B \mathcal{R}}{\eps_n} + \frac{\lambda_n}{\eps_n}R_{WD}(\theta_n) + \widehat{\TV}_{\eps_n,n} (f_n;\mu). 
    \end{align}
    The first term on the right hand side is bounded by $\frac{C}{\eps_n}\sqrt{\frac{\log(n)}{n}}$ with probability at least $1-\frac{1}{n^2}$ due to \Cref{lem: generalization error}.
    For justifying the application of this result, notice that the Barron semi-norms $[f_n]_\B$ are uniformly bounded and hence all $f_n$ and the limit $f$ are Lipschitz continuous with a uniform Lipschitz-constant.
    Moreover, for $n$ large enough, we must have $|f_n(0)| \leq |f(0)|+1$ as the contrary would contradict the convergence in $L^\infty(\Omega)$.  
    
    The second term can be bounded using the Lipschitz continuity of the loss function $\ell$ to obtain
    \begin{align}
        \frac{\mathcal{R}(f_n) - \inf_\B \mathcal{R}}{\eps_n} \leq \frac{C}{\eps_n} \|f_n-f\|_{L^1(\Omega)} \leq \frac{C}{\eps_n\sqrt{m_n}}
    \end{align}
    and the right hand side converges to zero due to our assumptions on the scaling of $m_n$ and $\eps_n$. 
    Hence, using the Borel--Cantelli lemma together with \Cref{prop: limsup regularizer} yields 
    \begin{align}
        \limsup_{n\to\infty} F_n(f_n) \leq \alpha [f]_\B + \TV(f;\mu) = F(f)
    \end{align}
    almost surely.
\end{proof}

Finally, we prove the compactness result \Cref{prop: compactness}. 

\begin{proof}[Proof of \Cref{prop: compactness}]
    For abbreviation, we denote $f_n \coloneq f_{\theta_n}$.
    Since $\alpha > 0$, from $\limsup_{n\to\infty} F_n(\theta_n)$ we can deduce that there exists a constant $C > 0$ such that $[f_n]_\B \leq R_{WD}(\theta_n) \leq C$.
    In particular, all functions $f_n$ are Lipschitz continuous with a uniform bound for the Lipschitz constants. 
    Proceeding by contradiction, suppose that $|f_{n_j}(0)| \to \infty$ along a subsequence $(f_{n_j})$.
    Combining this with the uniform Lipschitz bound implies that either $f_{n_j}\geq C$ on $\Omega$ or $f_{n_j} \leq -C$ for sufficiently large $j$.
    Consequently, \Cref{assum: loss} forces $\lim_{j\to\infty} F_{n_j}(f_{n_j}) = \infty$.
    This directly contradicts the hypothesis that $\limsup_{n\to\infty}F_n(f_n) < \infty$, thereby confirming the boundedness of $(f_n(0))$.
    Therefore, we can apply \citet[Lemma C.3]{park2023minimum} and obtain the existence of a function $f\in\B$ and a subsequence $(f_{n_j})$ such that $f_{n_j}\to f$ in $L^1(\Omega)$. 
\end{proof}

%%%%%%%%%%%%%%%%%%%%%%%%%%%%%%%%%%%

\section{Additional Numerical Illustrations}\label{sec: additional numerical experiments}

In this section, we provide additional details on the experiments presented already, and we also include some additional experiments.
We focus on the datasets from the main data, i.e., the two-dimensional toy dataset from \cref{fig:toy_model}---where Bayes classifiers are not unique and we can nicely illustrate the different decision boundaries induced by natural and total variation regularization / adversarial training---the CIFAR-10 dataset, as well as ImageNet.

\paragraph{Toy Dataset.}

% DGA: fonts could be bigger (no space issue here :))
\begin{figure}[htb]
    \centering
    \includegraphics[width=\linewidth]{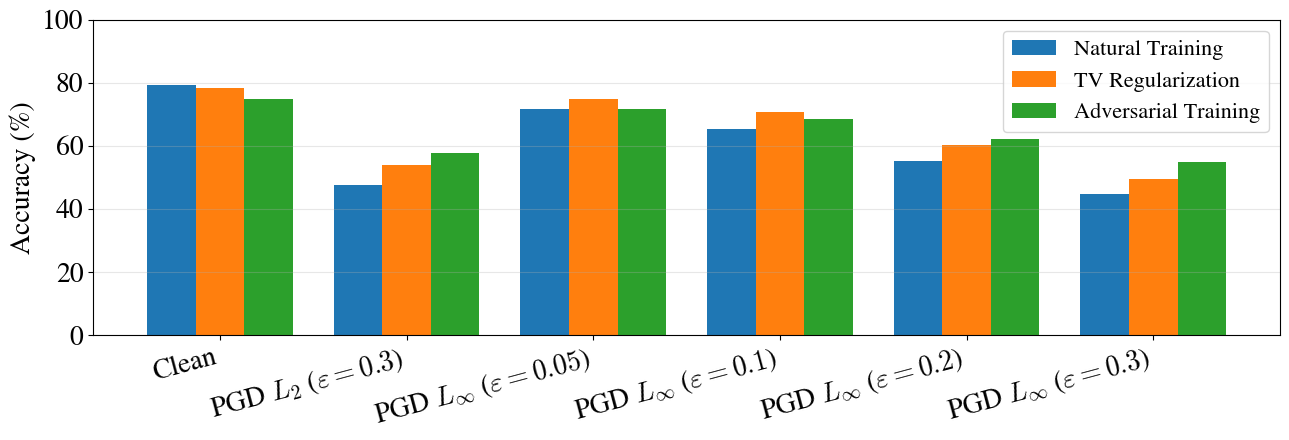}
    \caption{Clean and adversarial accuracies of the models from \cref{fig:toy_model}, trained on the toy dataset using the architecture \eqref{eq:neural network}.
    We used $n=1,000$ data points, $m_n=2,000$ neurons, $\eps_n = 0.24$, and $\lambda_n=0.01\cdot\eps_n$.
    The Bayes (i.e., best achievable) accuracy for this dataset is $80\%$.}
    \label{fig:robustness_bars}
\end{figure}

First, we revisit the toy example from \cref{fig:toy_model}.
Here, the probability distribution for the data points was chosen such that in the inner square $[-0.5,0,5]^2$ the probability of having label $0$ or $1$ is $50\%$ each, making Bayes classifiers non-unique in this set. 
Furthermore, in the set $[0.5,1]\times[-1,1]\cup[-0.5,0.5]\times[-1,-0.5]$ the probability for a data point to have label $1$ equals $90\%$.
In the remaining area that probability is $10\%$.
All Bayes classifiers equal $1$ in the former set, $0$ in the latter, and can take any number between $0$ and $1$ in the inner square $[-0.5,0.5]^2$.
Hence, this data distribution is particularly useful for evaluating how different training methods influence the decision boundary, particularly, in this set where it is undetermined by the data.

In \cref{fig:robustness_bars} we report the clean and adversarial accuracies of the three models from \cref{fig:toy_model}.
Note that, due to the overlap of the class conditional distributions, the best achievable accuracy on this dataset is $80\%$.
Notably, adversarial training suffers most from the well-known accuracy-robustness trade-off on clean data, see also the rightmost panel in \cref{fig:toy_model}. 
We observe that the classifier obtained by natural training has the highest clean accuracy, but as expected is most susceptible to all attacks considered here. 
The attacks themselves are standard; we consider one $\ell_2$-attack and several $\ell_\infty$-attacks, all computed by projected gradient descent (PGD), see \citet{madry2018toward}.
Moreover, TV regularization, i.e., problem \eqref{eq: full objective}, and standard adversarial training \citep{madry2018toward} produce comparable results. 
While TV regularization has better clean accuracy and is more robust to $\ell_2$ and small $\ell_\infty$ attacks, adversarial training slightly dominates for larger adversarial budgets. 
Note that we used to cross-entropy loss for these results, hence the mathematical equivalence of TV regularization and adversarial training \citep{bungert2023geometry} does not hold here.

Finally, in \cref{fig:l2_classifiers} (which should be compared to \cref{fig:toy_model}) we show the resulting decision boundaries when using $\ell_2$-adversaries instead of $\ell_\infty$. 
Our theory does not predict gradient sparsity in this situation with $\ell_2$-attacks, and indeed the resulting decision boundaries (just like the gradients) inside the undetermined set $[-0.5,0.5]^2$ are not axis-parallel, but rather approximately diagonal. 
Nevertheless, the decision boundaries of TV regularization and adversarial training are better behaved than that of natural training, as a consequence of the implicit regularization with $\int_\Omega\|\nabla f\|_2(\mu_0(x)+\mu_1(x))\de x + \alpha[f]_\B$ which we proved in \cref{thm: main}.

%%%%%%%%%%%%%%%%%%%%%%%%%%%%%%%%%%%%%%%%%%%%%%%%%%%%%%%%%%%%%%%%%%%%%%%

\begin{figure}[htb]
    \centering
    \begin{tikzpicture}
    % Contour plot
    \node[anchor=south west, inner sep=0] (plot) at (0,0) {
 \includegraphics[width=\linewidth,trim=0.4cm 2cm 0.3cm 2cm,clip]{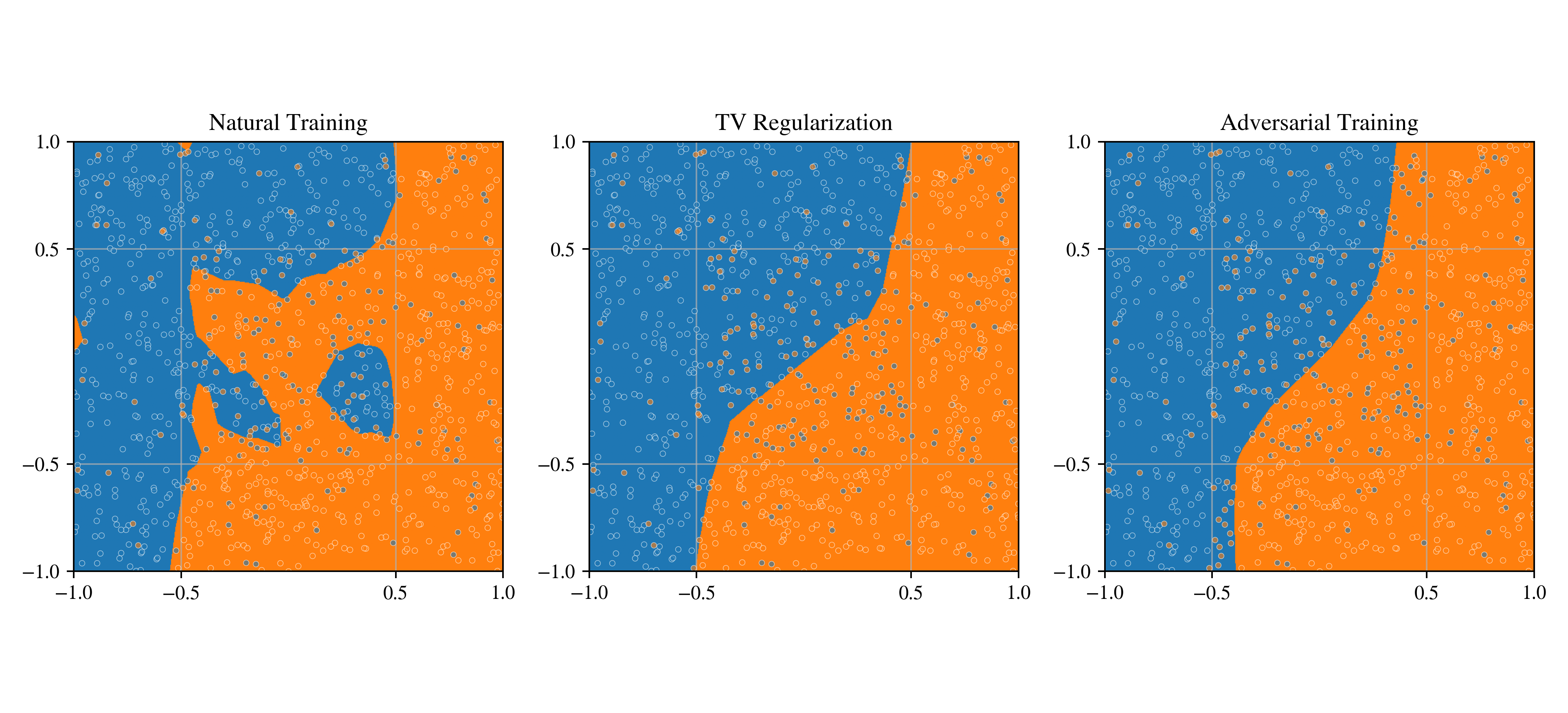}
    };
    % Gradient arrows
    \draw[->, very thick]
        (2.85,2.25) -- (2.5,2.6)
        node[midway, above right] {$\nabla f(x)$};    

    \draw[->, very thick]
        (7.,2.1) -- (7.3,1.7)
        node[midway, above right] {$\nabla f(x)$};    
    
    \draw[->, very thick]
        (11.7,2.1) -- (12.1,1.7)
        node[midway, above right] {$\nabla f(x)$};    

    % mark the points x
    \fill (2.85,2.25) circle (2pt)
        node[below left] {$x$};

    \fill (7.,2.1) circle (2pt)
        node[below left] {$x$};

    \fill (11.7,2.1) circle (2pt)
        node[below left] {$x$};    
\end{tikzpicture}
    
    \caption{Decision boundaries, using $\ell_2$-attacks instead of $\ell_\infty$-attacks:
    As expected, no gradient sparsity in the undetermined region $[-0.5,0.5]^2$ is observed, but the decision boundaries of TV regularization and adversarial training still have much lower complexity and are shorter than the one of natural training.}
    \label{fig:l2_classifiers}
\end{figure}

%%%%%%%%%%%%%%%%%%%%%%%%%%%%%%%%%%%%%%%%%%%%%%%%%%%%%%%%%%%%%%%%%%%%%%%

\paragraph{CIFAR-2 and Two-Layer Neural Networks.}

We provide additional numerical illustrations obtained by training two layer neural networks on two classes extracted from the CIFAR-10 dataset. 
Specifically, we trained the two-layer neural network architecture from \eqref{eq:neural network} with $4,096$ hidden neurons on the subset of CIFAR-10 that contains only the classes \textit{frog} and \textit{ship}.
We do this to reduce the problem to a binary classification problem, precisely as analyzed in \cref{thm: main}.
As for the toy example, we utilize three different training methods: natural training, TV regularization, and adversarial training. 

The results are displayed in \cref{fig:robustness_bars_CIFAR10} where we compare the achieved accuracy for clean and adversarially perturbed inputs and \cref{fig:sparsity_comparison_CIFAR10} for a thresholded sparsity comparison of the input gradients.
It can be observed that adversarial training, as well as TV regularization both result in better robustness and higher thresholded gradient sparsity, with the latter being even more affected by TV regularization. 
Both methods are comparable in term of adversarial robustness and yield higher robustness against the considered attacks than the naturally trained model.

%%%%%%%%%%%%%%%%%%%%%%%%%%%%%%%%%%%%%%%%%%%%%%%%%%%%%%%%%%%%%%%%

\begin{figure}[htb]
    \centering
    \includegraphics[width=\linewidth]{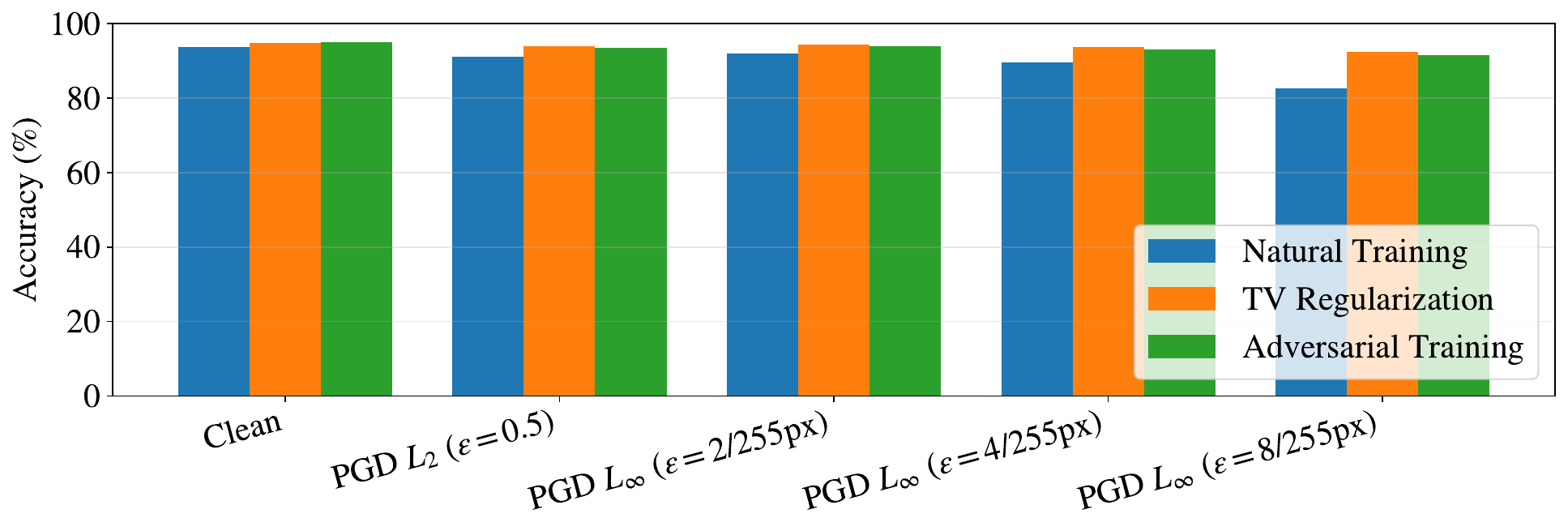}
    \caption{Clean and adversarial accuracies of the two-layer neural network models trained on two classes of CIFAR-10.
    We used $n=10,000$, $m=4,096$ neurons, $\eps_n =0.0587$, and $\lambda_n=0.0055\cdot\eps_n$.}
    \label{fig:robustness_bars_CIFAR10}
\end{figure}

\begin{figure}[htb]
    \centering
    \includegraphics[width=\linewidth]{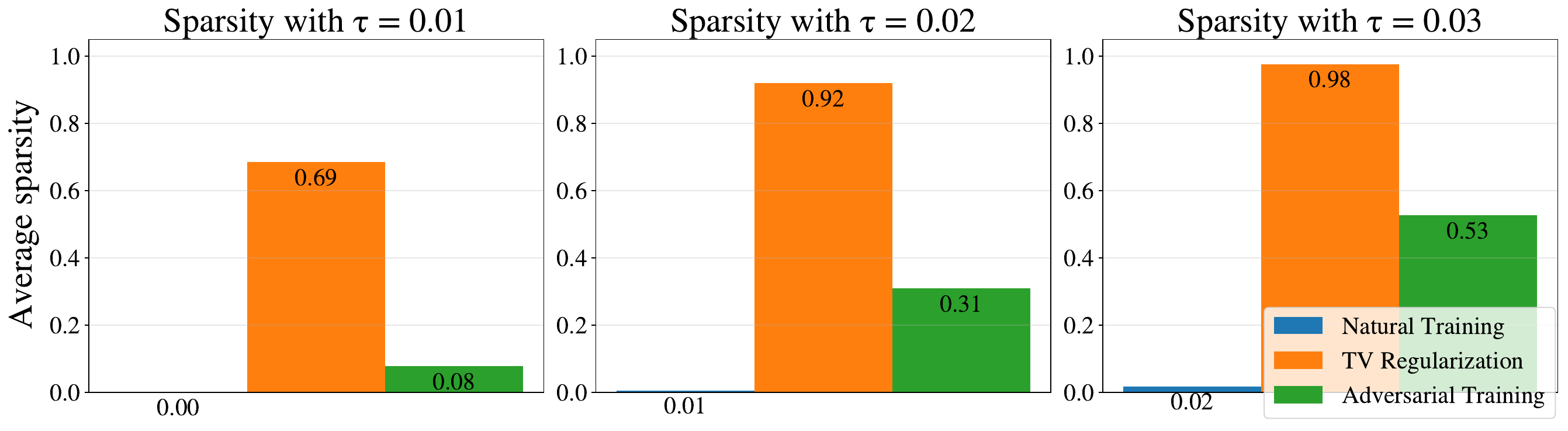}
    \caption{Sparsity of the saliency maps of the two-layer models. Larger value of the bar indicates more values below the threshold $\tau$.}
    \label{fig:sparsity_comparison_CIFAR10}
\end{figure}

%%%%%%%%%%%%%%%%%%%%%%%%%%%%%%%%%%%%%%%%%%%%%%%%%%%%%%%%%%%%%%%%%

\paragraph{ImageNet.}

In addition to the thresholded sparsity comparison in \Cref{fig:sweep}, we provide a histogram of the input gradient magnitudes to further analyze the networks' sensitivities.
Unlike the spatial visualizations in \Cref{fig:sensitivity_maps_varying_eps}, here the gradient magnitude is the absolute gradient value of every individual color channel.
Using the pre-trained models described in \cref{sec:numerical illustrations}, \Cref{fig:histogram} allows for a direct comparison of the raw gradient distributions between a naturally trained ResNet50 and its adversarially trained counterpart ($\eps = 1/255$) over a subset of the ImageNet validation dataset.
As already expected from the thresholded sparsity results, the adversarially trained model's gradient values are much more concentrated around zero than those of the naturally trained model.
In particular, this histograms reveals that the gradients are not strictly sparse, meaning that very few of the components are exactly zero, but rather that the majority of the gradient components have magnitudes extremely close to zero. 

%%%%%%%%%%%%%%%%%%%%%%%%%%%%%%%%%%%%%%%%%%%%%%%%%%%%%%%%%%%%%%%%
\begin{figure}[h]
    \centering
    \includegraphics[width=0.5\linewidth]{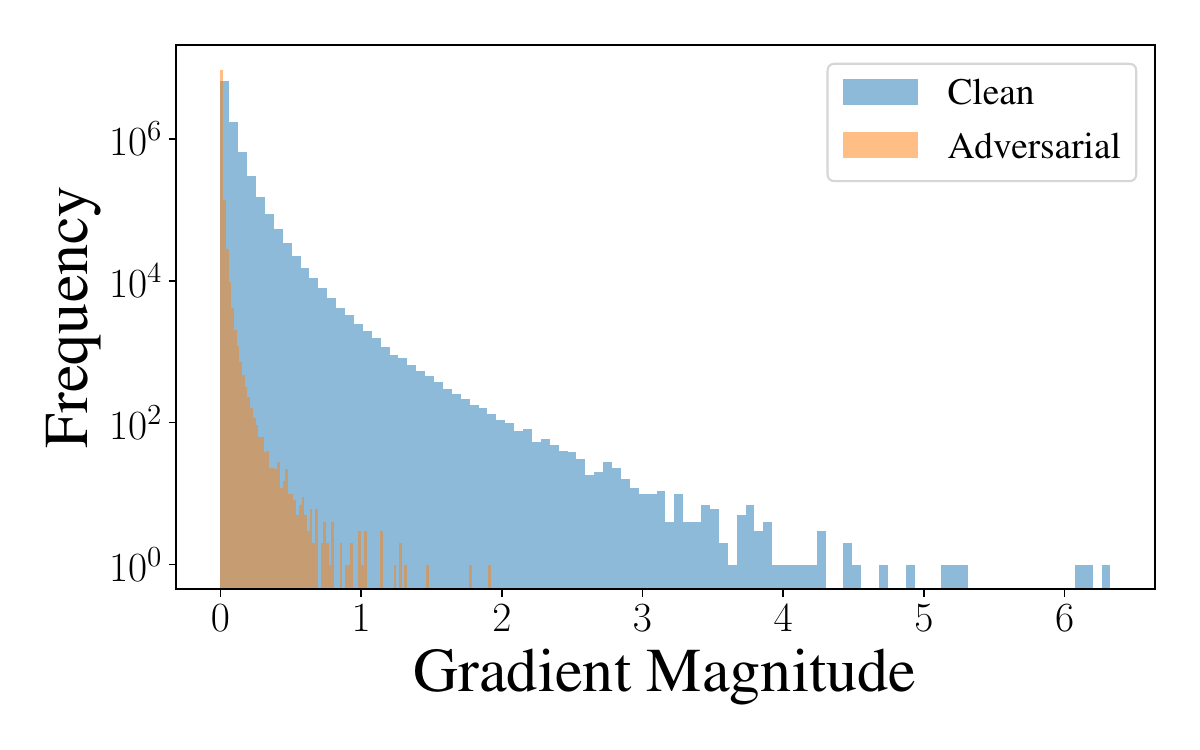}
    \caption{Distribution of input gradient magnitudes for naturally and adversarially trained ResNet50 models on the ImageNet dataset. Notably, the adversarially trained model exhibits a higher frequency of near-zero values.}.
    \label{fig:histogram}
\end{figure}

%%%%%%%%%%%%%%%%%%%%%%%%%%%%%%%%%%%%%%%%%%%%%%%%%%%%%%%%%%%%%%%%

Although our theory gives explanations for the apparent sparsity of saliency maps when computing  the gradient of the true class component (see our discussion on the multiclass case in \cref{sec: main result}) the empirical behavior of trained models for incorrect target classes remains of interest. 
To investigate this, we take the image from \cref{fig:sensitivity_maps_varying_eps} (true label: \textit{hummingbird}) and compute the gradients for two different classes. 
\cref{fig:wrong_label} visualizes these gradients with respect to the classes \textit{goldfish} (top row) and \textit{tennis ball} (bottom row).
Similar to the true-class setting, the naturally trained model produces unstructured noise, while the adversarially trained models generate gradients that are sparser and more structured.

%%%%%%%%%%%%%%%%%%%%%%%%%%%%%%%%%%%%%%%%%%%%%%%%%%%%%%%%%%%%%%%%%%%%%%%%

\begin{figure}[htb]
\centering
   \begin{subfigure}[b]{0.19\linewidth}
        \centering
        \includegraphics[width=\linewidth]{Figures/custom_original.pdf}
    \end{subfigure}
    \hfill
    \begin{subfigure}[b]{0.19\linewidth}
        \centering
        \includegraphics[width=\linewidth]{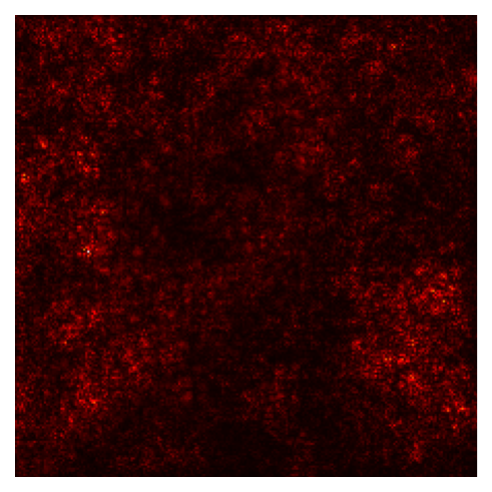}
    \end{subfigure}
    \hfill
    \begin{subfigure}[b]{0.19\linewidth}
        \centering
        \includegraphics[width=\linewidth]{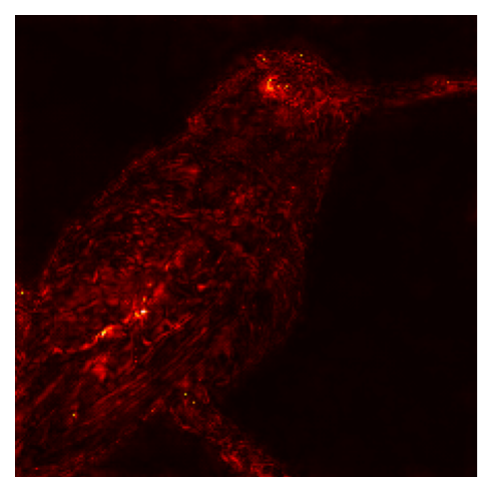}
    \end{subfigure}
    \hfill
    \begin{subfigure}[b]{0.19\linewidth}
        \centering
        \includegraphics[width=\linewidth]{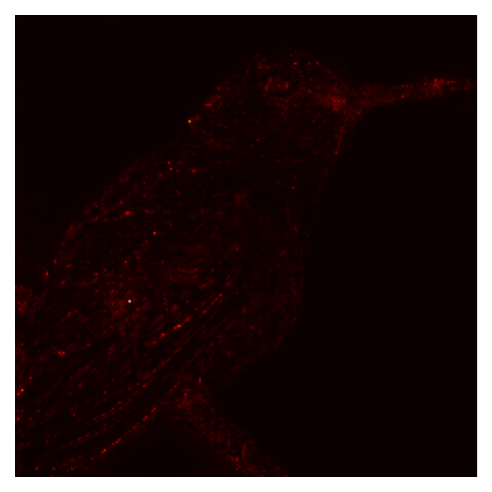}
    \end{subfigure}
    \hfill
    \begin{subfigure}[b]{0.19\linewidth}
        \centering
        \includegraphics[width=\linewidth]{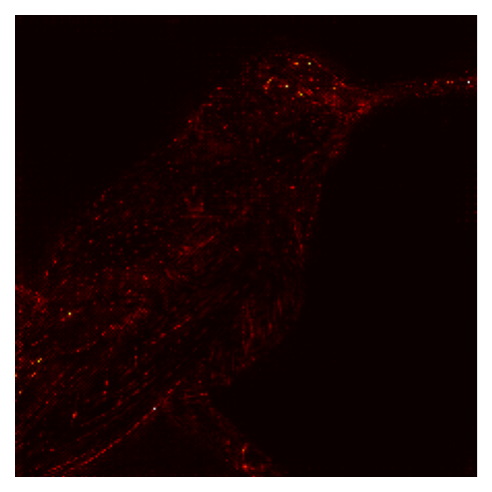}
    \end{subfigure}

    \begin{subfigure}[b]{0.19\linewidth}
        \centering
        \includegraphics[width=\linewidth]{Figures/custom_original.pdf}
        \caption{Original}
        \label{fig:sens_orig_off_target}
    \end{subfigure}
    \hfill
    \begin{subfigure}[b]{0.19\linewidth}
        \centering
        \includegraphics[width=\linewidth]{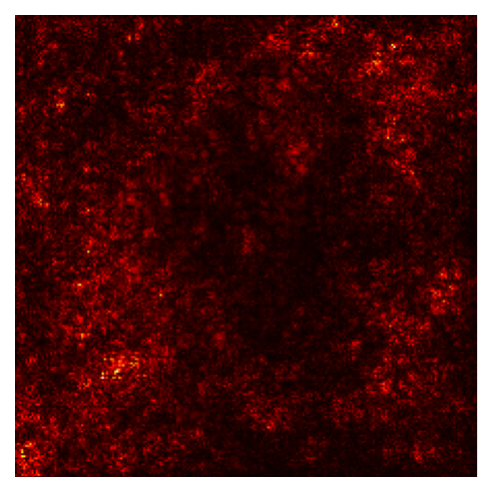}
        \caption{Natural training}
    \end{subfigure}
    \hfill
    \begin{subfigure}[b]{0.19\linewidth}
        \centering
        \includegraphics[width=\linewidth]{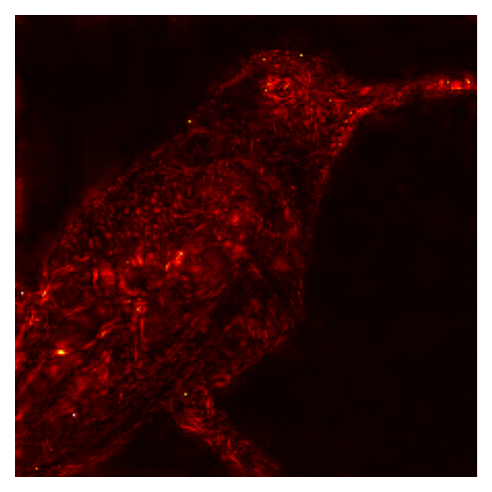}
        \caption{$\eps=0.5/255$}
    \end{subfigure}
    \hfill
    \begin{subfigure}[b]{0.19\linewidth}
        \centering
        \includegraphics[width=\linewidth]{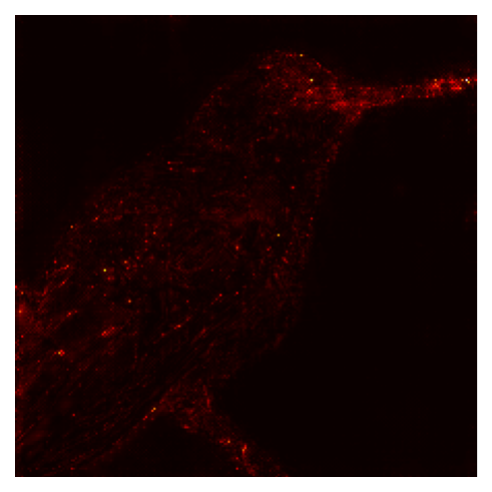}
        \caption{$\eps=1/255$}
    \end{subfigure}
    \hfill
    \begin{subfigure}[b]{0.19\linewidth}
        \centering
        \includegraphics[width=\linewidth]{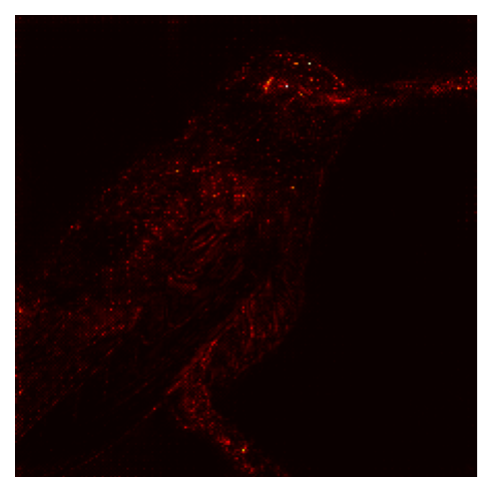}
        \caption{$\eps=2/255$}
    \end{subfigure}
    \caption{Off-target saliency maps across varying adversarial budgets. Left: The original image (ground truth: \textit{hummingbird}). Right: Input gradients from a naturally trained ResNet50 and adversarially trained models with increasing budgets. The gradients are computed with respect to incorrect target classes rather than the ground truth: \textit{goldfish} (top row) and \textit{tennis ball} (bottom row).}
    \label{fig:wrong_label}
\end{figure}

This observation aligns with the discussion of multiclass classification in \cref{sec: main result} if the class conditionals $\mu_c$ for different classes $c\in\{1,\dots,C\}$ sufficient overlap. 
Otherwise, this remains an interesting topic for future research.

To generate the saliency maps in \cref{fig:wrong_label,fig:sensitivity_maps_varying_eps}, we first compute the gradient of the respective class logit with respect to the input.
Specifically, we used the true class for \cref{fig:sensitivity_maps_varying_eps} and incorrect classes for \cref{fig:wrong_label}.
We then extract the maximum absolute gradient magnitude across channels for each pixel and applying a global min-max normalization to scale the resulting spatial map into the $[0,1]$ range for visualization.

\end{document}